\documentclass{article}
\usepackage{iclr2027_conference,times}

\usepackage{amsmath,amsfonts,bm}

\def\eqref#1{equation~\ref{#1}}

\def\1{\bm{1}}

\DeclareMathAlphabet{\mathsfit}{\encodingdefault}{\sfdefault}{m}{sl}
\SetMathAlphabet{\mathsfit}{bold}{\encodingdefault}{\sfdefault}{bx}{n}

\usepackage{hyperref}
\usepackage{url}

\usepackage{aas_macros}
\usepackage{graphicx}
\usepackage{subcaption}
\usepackage{xcolor}
\usepackage{wrapfig}
\usepackage{booktabs}
\usepackage{amsthm}
\newtheorem{theorem}{Theorem}

\iclrfinalcopy

\title{ENFR: Evidence Networks with Flow-matched References for Amortized Simulation-Based Model Inference}

\author{%
    Kai Lehman$^{1,2,3}$,
    Jameson Dong$^{4}$,
    Carolina Cuesta-Lazaro$^{2,5}$,
    Niall Jeffrey$^{6, 7}$,\\
    \bfseries
    Benjamin D. Wandelt$^{4,8}$,
    Francisco Villaescusa-Navarro$^{2,9}$,
    Sven Krippendorf$^{10}$,
    Jochen Weller$^{1,11,3}$
    \\
    $^1$Universit\"ats-Sternwarte M\"unchen, Fakult\"at f\"ur Physik,\\ Ludwig-Maximilians-Universit\"at, Scheinerstr. 1, 81679 M\"unchen, Germany \\
    $^2$Center for Computational Astrophysics, Flatiron Institute, 162 5th Avenue, New York, NY, 10010, USA\\
    $^3$Excellence Cluster ORIGINS, Boltzmannstr. 2, 85748 Garching, Germany \\
    $^4$Department of Physics and Astronomy, 
    Johns Hopkins University,\\
    3400 North Charles Street, Baltimore, MD, 21218, USA \\
    $^5$Center for Cosmology and Particle Physics, Department of Physics,
    New York University, \\
    726 Broadway, New York, NY 10003, USA\\
    $^6$ Department of Physics \& King’s Institute for Artificial Intelligence, King’s College London,\\ Strand, London WC2R 2LS, United Kingdom\\
    $^7$ Department of Physics \& Astronomy, University College London, Gower Street, London, WC1E 6BT, UK \\
    $^8$Department of Applied Mathematics and Statistics, Johns Hopkins University,\\ 3400 North Charles Street, Baltimore, MD, 21218, USA\\
    $^9$Department of Astrophysical Sciences, Princeton University, Princeton, NJ 08544, USA\\
    $^{10}$Cavendish Laboratory and Department of Applied Mathematics
    and Theoretical Physics, \\
    University of Cambridge, Cambridge CB3 0WA, United Kingdom\\
    $^{11}$Max-Planck-Institut f\"ur extraterrestrische Physik, Giessenbachstr.~1, 85748 Garching, Germany\\
    \texttt{kai.lehman@physik.lmu.de}
}

\begin{document}

\maketitle

\fancyhead[L]{Preprint}

\begin{abstract}
    The evidence, $\mathcal{Z} = p(x | \mathcal{M})$, of a data-generating model $\mathcal{M}$ is a powerful tool in Bayesian frameworks: it enables model inference, model misspecification tests, and goodness-of-fit assessment as an alternative to standard discrepancy statistics, e.g. $\chi^2$. While versatile, the evidence is notoriously difficult to compute, even when the likelihood is available. In both likelihood-based and simulation-based inference, estimating $\mathcal{Z}$ provides an immediate solution to high-dimensional goodness-of-fit or model misspecification tests: using $\mathcal{Z}(x)$ as the test statistic induces a minimum-volume level-$\alpha$ test. Such tests are essential to assess model fidelity, which is necessary to trust simulation-based parameter inference. Existing learned approaches largely target Bayes factors, which only measure evidence ratios. Instead, we provide accurate, fast, and amortized estimates of the absolute evidence. We present ENFR: \textit{Evidence Networks with Flow-matched References}. Building on the Evidence Network framework of \citet{jeffrey_24}, the proposed method first learns Bayes factors with respect to a flow-matched reference, from which the absolute evidence is recovered. We show improvement in estimation of $\mathcal{Z}$ over the flow-based baseline in all applications considered and the best performance of all density estimation methods considered in our benchmarks. We furthermore introduce and provide the proof for an architecture that jointly learns the individual evidences of a set of models $\{\mathcal{M}_i\}$. We demonstrate the performance and scaling behavior of these approaches on data from a variety of scientific applications.
\end{abstract}

\section{Introduction}
Science is fundamentally tasked with modeling nature from first principles, and increasingly this modeling is done with high-fidelity simulators that produce samples from a likelihood that cannot itself be evaluated~\citep[e.g.][]{geant_03,mesinger_11,villaescusa_21}. Simulation-based inference (SBI) has changed how such models are confronted with data: instead of requiring a tractable likelihood, it enables likelihood-free inference directly from simulator samples~\citep[][]{papamakarios_16,alsing_18,brehmer_20,cranmer_20}. These methods have already unlocked information in the data that was previously inaccessible~\citep[e.g.][for recent scientific applications]{wietersheim_25,thomsen_26}, and are being readily adopted across scientific disciplines.

The focus of SBI so far has been parameter estimation. What we ultimately care about in science, however, is assessing the relative fits of models when confronted to data. In a Bayesian framework this is measured by the evidence $p(x|\mathcal{M})$, which is the probability that the observed data $x$ has been produced by a physical or statistical model $\mathcal{M}$. In parameterized settings it is common to express the evidence as the marginal likelihood using the model parameters $\theta$:
\begin{equation}
    p(x|\mathcal{M}) = \int p(x|\theta,\mathcal{M}) p(\theta|\mathcal{M}) d\theta,
\end{equation}
where $p(\theta|\mathcal{M})$ denotes the prior on the parameters and $p(x|\theta,\mathcal{M})$ is the model likelihood. The evidence serves this purpose in two ways. For model comparison, we are interested in the Bayes factor $\log K= \log\frac{p(x|\mathcal{M})}{p(x|\mathcal{M}_0)}$ with respect to some reference model $\mathcal{M}_0$. Beyond model comparison, the absolute evidence enables powerful tests of the data such as model misspecification tests, and an alternative to $\chi^2$ goodness-of-fit tests in high dimensions that does not require summary statistics.

Estimating the evidence has not been the focus of SBI, and the methods to do so are correspondingly less developed. Traditionally, computing the evidence has been a daunting task as the naive estimation of the integral is quickly intractable in high dimensions. Computationally heavy methods such as nested sampling~\citep[e.g.][]{feroz_09,handley_15,speagle_20} have been developed for these cases, but they require the ability to evaluate the likelihood, and even then have seen limited use in scientific applications due to computational infeasibility. How to evaluate the model evidence and Bayes factors when the likelihood is intractable but a simulator is available has remained a largely underexplored topic.

\cite{jeffrey_24} presented the application of neural network based density ratio estimators to learn Bayes factors in this regime. We build on this previous work and introduce ENFR, Evidence Networks with Flow-matched References, which combines flow-based models and evidence networks: using a generative flow as the reference model, we obtain not only Bayes factors but precise estimates of the absolute evidence. We furthermore introduce a new loss function to jointly estimate the Bayesian evidence and Bayes factors of multiple models in Section~\ref{sec:methods}.

Alongside the method, we introduce an evaluation framework built on scientifically relevant applications from a range of disciplines in Section~\ref{sec:examples}. It includes cases with a known likelihood, where the evidence can be estimated with nested sampling and our estimates can therefore be tested accurately against a ground truth, as well as simulation-based cases where only the calibration of the Bayes factors can be assessed. Across these applications, ENFR performs best of all learned estimators considered.

\section{Related work}

\paragraph{Generative models for evidence estimation} Flow based generative models can both obtain samples from the target distribution and evaluate its density. 
In recent years, generative models such as flow matching~\citep{lipman_23} or score-based diffusion~\citep{song_19,song_21} have been successfully pushed into very high-dimensional regimes. However, their primary success criterion is often sample quality. An exact evaluation of the density is often not an explicit goal and its quality is not a main concern. In cases where the estimated density is evaluated as a metric, it is usually maximized, without a ``true'' value to aim for or any interpretable measure for its quality. For evidence estimation, however, the evaluation of the density is the final goal without any concerns of sample quality. Given the optimization objective in those two generative models, the velocity field and the score, it seems clear that these models do not prioritize the evaluation of the likelihood. 

\paragraph{Flow matching} In this work we consider flow matching~\citep[][]{lipman_23}, learning the velocity field $v_\phi = \frac{dx}{dt}$ when mapping from samples of a known noise distribution $z\sim p(z)$ to data samples $x$ with learnable parameters $\phi$ and $t$ parameterizing the path. Throughout this work we consider linear interpolants for the flow, i.e. $x_t = t x + (1-t) z$. The density can then be obtained via integration from the continuity equation (see Appendix~\ref{app:flow}).

\paragraph{Evidence Networks}

Given a reference model $\mathcal{M}_0$, one can use classifiers to learn a density ratio, or equivalently, the Bayes factor $\log K$ \citep{jeffrey_24}. This can be achieved through a variety of loss functions~\citep{masnadi_08}.  \citet{jeffrey_24} recommend an exponential loss function which is given by:
\begin{equation}
    \mathcal{L}(x,m) = e^{(\frac{1}{2}-m) \mathcal{J}_\alpha(f_\psi(x))}
\end{equation}
where $m\in\{0,1\}$ denotes the model label (e.g. $m=1$ if $x\sim p(x|\mathcal{M})$ and $m=0$ if $x\sim p(x|\mathcal{M}_0)$) and $f_\psi$ is a neural network with trainable parameters $\psi$. $\mathcal{J}_\alpha$ denotes the l-POP transform defined as $\mathcal{J}_\alpha(x) = x + \mathrm{sgn}(x) |x|^\alpha$
which provides for more accurate computation of extreme evidence ratios than cross-entropy loss. We explicitly show the minimizer of this loss to be the Bayes factor in Appendix~\ref{app:lpop}. The hyperparameter $\alpha$ can be chosen freely but is recommended by~\cite{jeffrey_24} to be set to $\alpha=2$. We follow this recommendation and will refer to this architecture as an evidence network (EN). As we are fundamentally interested in evidences as opposed to Bayes factors, we need to choose a reference model with known density to reconstruct an evidence: $\log q_\psi(x|\mathcal{M}) = \log K_\psi + \log p(x|\mathcal{M}_0)$.\newline
The choice of the reference model can be critical for good evidence estimation: evidence networks suffer from a well-known failure mode in the case where the distance between the two distributions $p(x|\mathcal{M})$ and $p(x|\mathcal{M}_0)$ is very large~\citep{rhodes_20}. This is known in the literature as the density chasm problem. In the following we describe approaches to mitigate this issue.

\paragraph{Telescoping ratio estimation}

As a solution to the density chasm problem, an iterative procedure has been proposed by \cite{rhodes_20}, in which intermediate waymark distributions are constructed, to produce an array of reference distributions and to learn the Bayes factors (or equivalently, density ratios) between consecutive waymark distributions. Choosing a known noise density $p(z)$, enables the reconstruction of the initial density. While~\cite{rhodes_20} have applied this also to energy based modelling with a flow-based noise distribution they do not consider controlled density-estimation benchmarks in which the reconstructed target density can be compared directly to its known normalized pointwise density. Furthermore our work here differs in architecture.

\paragraph{Time Score Matching (CTSM)} A related continuous formulation is given by~\citet{choi_22,yu_25,chen_25}, in which neural networks learn the time score $s(x,t) = \partial_t \log p(x_t)$. These models are in training similar to flow-matching, i.e. the target and noise distribution are connected with some interpolant.
The learned time score can then be integrated to obtain a density estimate: $\log q(x) = \log p(z) + \int_0^1 s(x,t)\, \mathrm{d}t$.
We train these models using the conditional time score matching objective (CTSM)~\citep[for more details see][]{yu_25} with either the variance preserving interpolant (VP) or the Schr\"odinger bridge interpolant (SB) (we show the interpolants in Appendix~\ref{app:ctsm}). We note that while these methods have been primarily formulated for density ratio estimation, we evaluate them here as density estimators since we are fundamentally interested in evidences rather than Bayes factors. In our experiments, we have found the variance of the target to be very high close to the data distribution and to be limiting for the performance of this method. We also note, that under imperfect training, the resulting probability density function is not guaranteed to be normalized as opposed to flow-based generative models. In our numerical experiments we compare with CTSM-VP, and CTSM-SB with a learned flow matching reference distribution, which we call CTSM-SB + FlowRef.

\paragraph{Featurized Density Ratio Estimation} Another related approach is featurized density ratio estimation~\citep{choi_21} which uses an invertible flow to learn a common feature space of samples from two densities and to approximate the density ratio in this feature space. Again our work differs as we are not only interested in density ratios, but also actual densities. 

\paragraph{Flow contrastive estimation} \cite{gao_19} propose flow contrastive estimation, which jointly trains an energy-based model and a normalizing flow through a noise-contrastive binary-classification objective, with the flow serving as an adaptive noise distribution. The flow is brought iteratively closer to the data distribution. While they do present density estimation results, they do not fully benchmark their density estimation for pointwise ground-truth examples beyond two dimensions. They furthermore show that maximizing the likelihood and improving sample quality do not necessarily go together. Their method differs from ours both in architectural setup and training method.

\section{ENFR: Evidence Networks with Flow-Matched References} \label{sec:methods}
\begin{figure}[h]
\begin{center}
    \includegraphics[width=\textwidth]{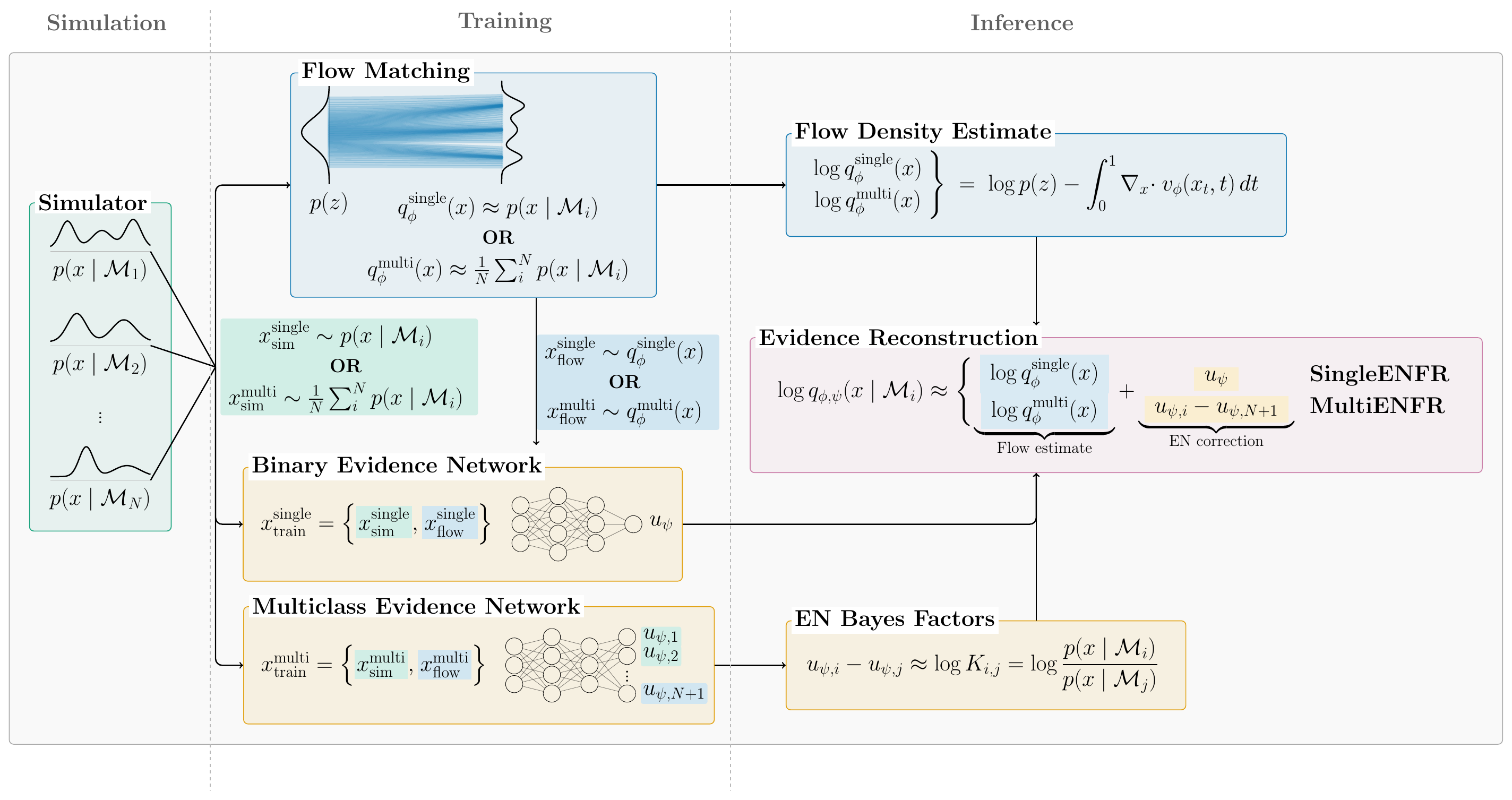}
\end{center}
\caption{A visualization of our proposed SingleENFR and MultiENFR architectures. Given simulated data from one or multiple models, we train flow matching once over all models. In a second step, we use the simulated data and samples from the flow to train an evidence network classifier, which learns all Bayes factors. Those can be combined with the evaluated flow density to obtain the evidence under a given model.}
\label{fig:cartoon}
\end{figure}

\paragraph{Evidence Networks with Flow-matched References: ENFR}
Given that ENs are a powerful method for the estimation of the evidence if a fitting reference model is present, we propose to learn such a reference using a generative model (see Figure~\ref{fig:cartoon}). We choose flow matching for this task. The EN with learnable parameters $\psi$ then only has to learn a correction $\log K_\psi$ from the flow estimate to the true underlying density by discriminating simulations from flow samples:
\begin{equation}
    \log p(x|\mathcal{M})= \log q_\phi(x|\mathcal{M}) + \log \frac{p(x|\mathcal{M})}{q_\phi(x|\mathcal{M})} \approx \log q_\phi(x|\mathcal{M}) + \log K_\psi.
\end{equation}
In this setup, we can benefit from the fact that we only need accurate sampling from the flow to generate half of the training data for the evidence network. The flow density has to be evaluated only at inference time. We call this setup SingleENFR.

We show that when the flow is already a very good fit to the data and no informative density-ratio correction is available, the population-optimal constant correction is zero under balanced training in Appendix~\ref{app:en_const}. While we could use any evidence network loss function for the estimation of $\log K$, we follow the recommendation of~\cite{jeffrey_24} and choose l-POP. We note that in a similar fashion, one can use a pre-trained flow as the base distribution for CTSM using the SB bridge which has been used in~\cite{chen_25}. We also compare to that case in our experiments, calling it CTSM-SB + FlowRef.

\paragraph{ENFR with multiple models (MultiENFR)}

We further introduce a more streamlined way to obtain an array of Bayesian evidences for a set of multiple models based on the l-POP loss function which we call MultiENFR (see Figure~\ref{fig:cartoon}). Given $N$ such models, we introduce the pairwise loss:
\begin{equation}\label{equ:multiloss}
    \mathcal{L}(x, m) = \frac{1}{N-1} \sum_{i\neq m} e^{-\frac{1}{2}(\mathcal{J}_\alpha(f_{\psi,m}(x))-\mathcal{J}_\alpha(f_{\psi,i}(x)))}.
\end{equation}
\begin{theorem}[Asymptotic recovery of Bayes factors]
\label{thm:caten_bayes_factor}
Consider \(N\geq2\) models \(\{\mathcal M_i\}_{i=1}^N\) with evidences \(p_i(x)\equiv p(x\mid\mathcal M_i)\), sampled with equal probability \(p(\mathcal M_i)=1/N\). Let \(f_{\psi,i}(x)\) denote the \(i\)-th output of the multiclass evidence network and let \(\mathcal J_\alpha\) be a fixed scalar link function, with transformed outputs \(u_{\psi,i}(x)\equiv \mathcal J_\alpha(f_{\psi,i}(x))\).

Suppose the network is trained with the pairwise exponential loss
\(\mathcal L(x,m)=\frac{1}{N-1}\sum_{j\neq m}\exp[-\frac12(u_{\psi,m}(x)-u_{\psi,j}(x))]\) (Equation~\ref{equ:multiloss}).
Assume that the densities \(p_i(x)\) are positive almost everywhere on the common support, that the population risk is finite, that the transformed network class is sufficiently expressive to represent \(u_i(x)=\log p_i(x)+c(x)\) for some model-independent function \(c(x)\), and that asymptotic training reaches a global minimizer of the population risk.

Then every global population-risk minimizer satisfies \(u_i^*(x)=\log p(x\mid\mathcal M_i)+c(x)\) almost everywhere, where \(c(x)\) is independent of the model index. Consequently, for any pair of models \(\mathcal M_i\) and \(\mathcal M_j\),
\begin{equation}
\label{eq:caten_bayes_factor}
\mathcal J_\alpha\!\left(f_{\psi^*,i}(x)\right)
-
\mathcal J_\alpha\!\left(f_{\psi^*,j}(x)\right)
=
\log\frac{p(x\mid\mathcal M_i)}
         {p(x\mid\mathcal M_j)}
=
\log K_{ij}(x).
\end{equation}
\end{theorem}

\begin{proof}
See Appendix~\ref{app:caten}.
\end{proof}

In order to handle the density chasm problem, we introduce one more class, which is given by synthetic data samples from a flow $q_\phi^\mathrm{multi}$ with trainable parameters $\phi$. The generating flow is pre-trained on simulated data from all the other models:  $x\sim \frac{1}{N} \sum_{n=1}^N p(x|\mathcal{M}_n)$. We note that this is a fundamentally different flow density target than in SingleENFR, as the result is no longer an evidence estimate in itself. It is rather an estimate of the mixture of the evidences of many models. We use this flow reference as an additional model in our multiclass evidence network: there are therefore $N+1$ classes to discriminate. We generate an equal share of training data from this pretrained flow. The learned Bayes factor with respect to the flow class along with the evaluated flow density then delivers the evidence $\log p(x|\mathcal{M}_i) = \log K_{i,N+1} + \log q_\phi(x)$.

\section{Results} \label{sec:examples}

In this section we benchmark the performance of our proposed ENFR architecture and compare them to other learned evidence estimates on a variety of different scientific applications.

\subsection{Analytical evidence: Gaussian mixture model}

We begin with a Gaussian mixture model (GMM) where the true evidence is analytically known and we compare root mean square errors (RMSE) of our simulation-based estimates with respect to the ground truth. In the known ground truth case we generate data from a $d$‑dimensional correlated Gaussian mixture model (GMM) with three components. The details of the experimental setup are given in Appendix~\ref{app:GMM}. 

Figure~\ref{fig:scaling} shows that in our scaling tests the SingleENFR setup consistently outperforms the other tested models. We also find that CTSM is highly outperformed by all other methods. Using CTSM with a flow matched reference model, does not improve the flow significantly in all combinations tested here. A comparison of the RMSE values of the evidence estimates with respect to the ground truth is given in Table~\ref{tab:GMM}. We also find that our model learns only minor improvements starting at 512 dimensions. \newline

\begin{figure}
\centering

\makebox[\textwidth][c]{%
\begin{subfigure}[b]{.4\textwidth}
    \centering
    \includegraphics[width=\textwidth]{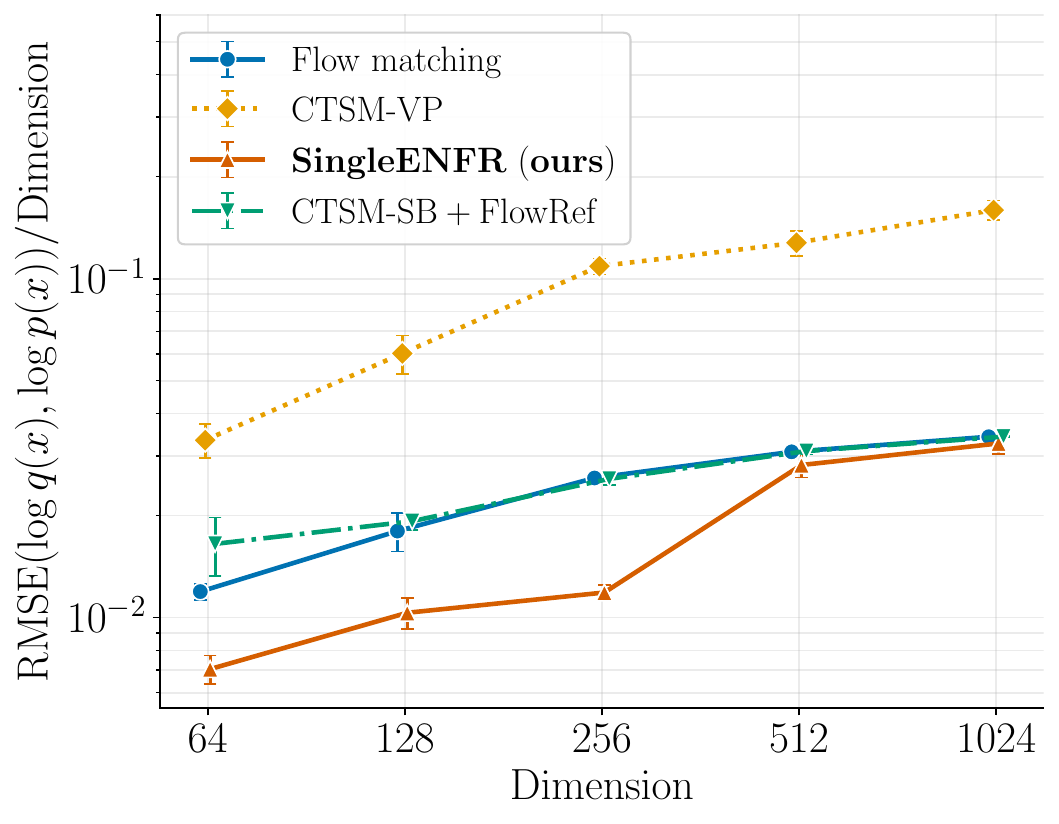}
    \caption{}
\end{subfigure}
\hspace{0.06\textwidth}
\begin{subfigure}[b]{.4\textwidth}
    \centering
    \includegraphics[width=\textwidth]{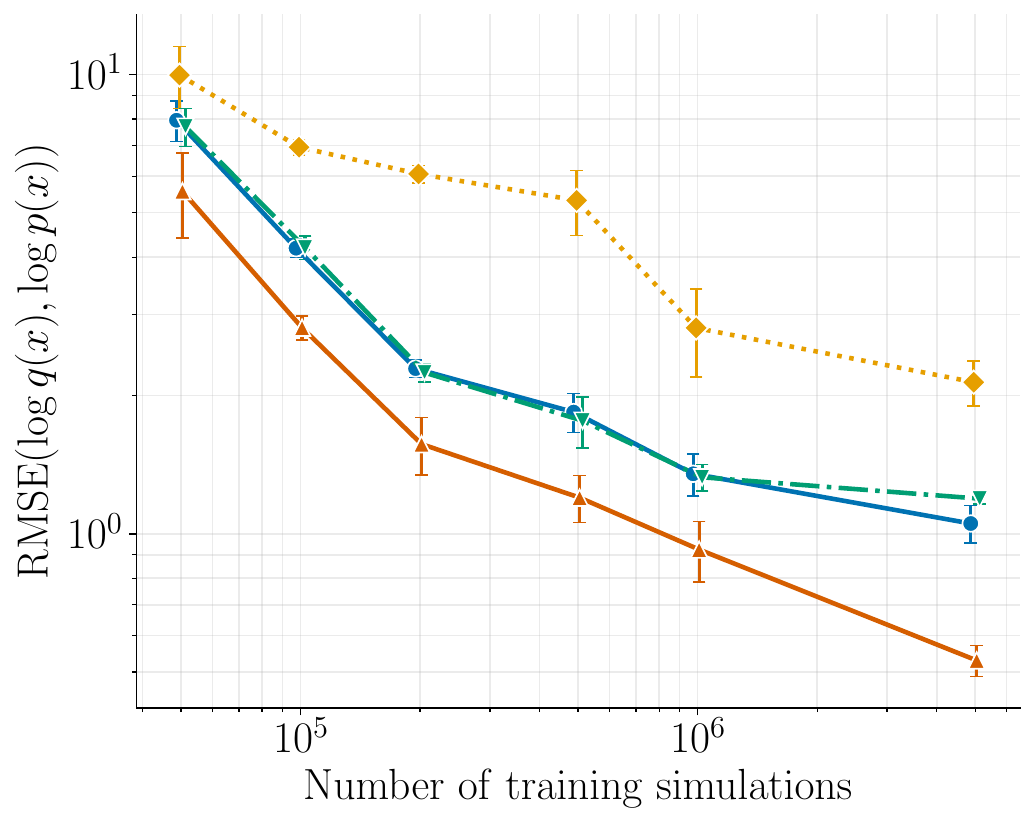}
    \caption{}
\end{subfigure}%
}
    \caption{Scaling of density-estimation accuracy with (a) data dimensionality and (b) number of training simulations on the Gaussian mixture model of Appendix~\ref{app:GMM}, for which the true density is known analytically. We report the RMSE between estimated and true log-density over $200$ held-out test points. In (a) the simulation budget is fixed at $5\times10^6$; in (b) the dimensionality is fixed at 64. The line shows the mean across three trained models and the errorbars show the estimated standard deviation. SingleENFR (orange) improves on the flow it is built on (blue) at every point and is the most accurate of all estimators considered, while CTSM-VP (yellow) is consistently the least accurate.}
    \label{fig:scaling}
\end{figure}

\subsection{Analytical likelihood applications}

In this section we present examples where the true underlying likelihoods are known, such that the evidences can be estimated with established likelihood-based frameworks. We benchmark the models by comparing the simulation-based evidences to the likelihood-based estimates via the root mean square error (RMSE) of the log probabilities. For our likelihood-based estimates we run nested sampling using the dynesty library~\cite{speagle_20}.

\begin{table}[t]
\caption{RMSE of our learned log-evidence estimates against the likelihood-based log-evidence estimate over 100 held out BAO test simulations (lower is better). Our Models SingleENFR and MultiENFR consistently outperform both the flow and CTSM-SB with a flow reference.}
\label{tab:bao_sims}
\begin{center}
\begin{tabular}{lcccc}
\toprule
Simulation & flow & SingleENFR (ours) & MultiENFR (ours) & CTSM-SB + FlowRef  \\
\midrule
$\Lambda$CDM 
& 1.34 & 0.53 & \bf{0.52} & 2.74  \\

$w$CDM 
& 1.40 & 0.76  & \bf{0.67}  & 2.46  \\

$w_0w_a$CDM 
& 1.16 & 0.83 & \bf{0.81} & 3.79  \\

$\nu$CDM 
& 1.41 & 0.47 & \bf{0.46} & 1.77  \\

$\Omega_k$CDM 
& 1.52 & \bf{0.54} & 0.78 & 2.29   \\

\bottomrule
\end{tabular}
\end{center}
\end{table}

\subsubsection{Baryonic acoustic oscillations}

We further test on a scientific example from cosmology. In the early Universe, sound waves propagated through the tightly coupled photon baryon plasma, driven by the competition between gravity and radiation pressure. When photons decoupled from baryons, these waves froze in and imprinted a characteristic length scale, the sound horizon $r_d$, on the correlations between cosmic structures. Baryonic acoustic oscillations (BAO) thus provide a \emph{standard ruler}: measuring its apparent size along and across the line of sight at different redshifts constrains the expansion history $H(z)$, and therefore the nature of the dark energy that drives the late-time accelerated expansion.

In the standard cosmological model, $\Lambda$CDM, dark energy is a cosmological constant with a fixed equation of state $w=-1$. Whether this description is sufficient is currently an open question, and different extensions have been proposed: a constant but free equation of state ($w$CDM), a time-evolving one parameterized as $w(a)=w_0+w_a(1-a)$ ($w_0w_a$CDM), as well as one-parameter extensions that free the neutrino mass sum ($\nu$CDM) or the spatial curvature ($\Omega_k$CDM). Recent BAO measurements from DESI ~\cite{desi_25} have renewed interest in dynamical dark energy, making a principled comparison of these models a timely application. We compare the evidences of these five models for BAO data vectors (more details in Appendix~\ref{app:BAO}).

Table~\ref{tab:bao_sims} shows the resulting evidences and Bayes factors on 100 held out test simulations for the different cosmological models.  
We find that our models SingleENFR and MultiENFR are consistently outperforming both the flow, and CTSM-SB with the flow reference. 
Table~\ref{tab:bao_data} shows the evidence estimates of our best model, MultiENFR, on the state-of-the-art measurement DESI DR2~\citep{desi_25}. The estimates on the data agree on the order of the nested sampling error bars.

\begin{table}
	\centering
	\caption{Evidence estimates using MultiENFR on the DESI DR2 observed data.  For MultiENFR we also quote the estimate of the standard deviation over the ensemble of 5 members. We can reproduce the likelihood-based evidence estimates of the DESI DR2 data with our simulation-based framework within the ensemble scatter.}
    \label{tab:bao_data}
	\begin{tabular}{lcccc} 
        \toprule
		Model &  Likelihood-based analysis & MultiENFR &  $\Delta$ & MultiENFR $\log K$\\
        \midrule
        $\Lambda\mathrm{CDM}$ &  -13.99$\pm$0.16  & -14.07$\pm$0.34 & -0.08 & 0.00 \\
        $w\mathrm{CDM}$ &  -16.41$\pm$0.18& -16.00$\pm$0.35 & 0.40 & -1.94 \\
        $w_0w_a\mathrm{CDM}$ &  -15.63$\pm$0.18 & -15.46$\pm$0.40 & 0.17 & -1.40 \\
        $\nu\mathrm{CDM}$ &  -14.06$\pm$0.16 & -14.01$\pm$0.30 & 0.04 & 0.06 \\
        $\Omega_k\mathrm{CDM}$ &  -15.88$\pm$0.17 & -15.72$\pm$0.23 & 0.16 & -1.65  \\
        \bottomrule
	\end{tabular}
\end{table}

\subsubsection{Time series from an ODE}

In this example we apply our method to a 64 dimensional time series from an ordinary differential equation (ODE). The data consists of noisy measurements of the value $x(t)$ at the 64 evenly spaced time steps. We consider the harmonic ($\beta=0, \gamma =0$), damped($\beta = 0, \gamma \neq 0$), and Duffing ($\beta\neq 0, \gamma \neq 0$) oscillators:

\begin{equation}
    \frac{\mathrm{d}^2 x}{\mathrm{d}t^2} = - 2 \gamma \frac{\mathrm{d}x}{\mathrm{d}t} -\omega^2 x - \beta x^3.
\end{equation}

At each time step we add noise $\epsilon \sim \mathcal{N}(0,(0.1)^2)$. The priors are given in Table~\ref{tab:ode_priors}. 

Figure~\ref{fig:ode_true_pred} shows that our MultiENFR model outperforms all other models tested here. Both MultiENFR and SingleENFR evidence estimates are significantly more precise than those of the flow. We present our full results in Table~\ref{tab:ode_rmse}, where we find again that the combination of CTSM-SB with a flow reference has the highest RMSE, i.e. the worst performance of the models tested.

\subsubsection{Diffusion drift SDE}

We present an application from neuroscience, modeling the decision process of a brain. We integrate an SDE that collects information about a particular state of the world:

\begin{equation}
    \mathrm{d}x = (v + \lambda x) \mathrm{d}t + \mathrm{d}W,
\end{equation}

where we differentiate between the constant drift ($\lambda=0$) and leaky drift ($\lambda < 0$) models. The prior is chosen such that the models are not nested and are shown in Table~\ref{tab:ddm_priors}. We integrate until we hit a collapsing symmetric boundary $B(t) =\pm B e^{-\kappa t} $. We record the reaction time of hitting this boundary as the result of the run, with $\pm$ indexing whether the positive or negative boundary was hit. We also add the non-decision time as an additional parameter to the recorded decision time, which accounts for delays in the decision making process. For each parameter combination, we simulate 400 times and record the reaction times. The data vector is the set of those 400 times. In order to respect exchangeability, we use permutation equivariant / invariant neural networks for the flow and the EN / CTSM. 

We show our results for the flow, SingleENFR and MultiENFR in Figure~\ref{fig:ddm_pred_true}. We find that the flow alone can already produce precise estimates of the evidence for both the constant drift and the leaky drift model. Adding the SingleENFR correction however again improves the resulting estimates in RMSE. Our best model is MultiENFR, which produces the highest precision in evidence estimates.

\begin{figure}[h]
\begin{center}
\begin{subfigure}[b]{.32\textwidth}
    \includegraphics[width=\textwidth]{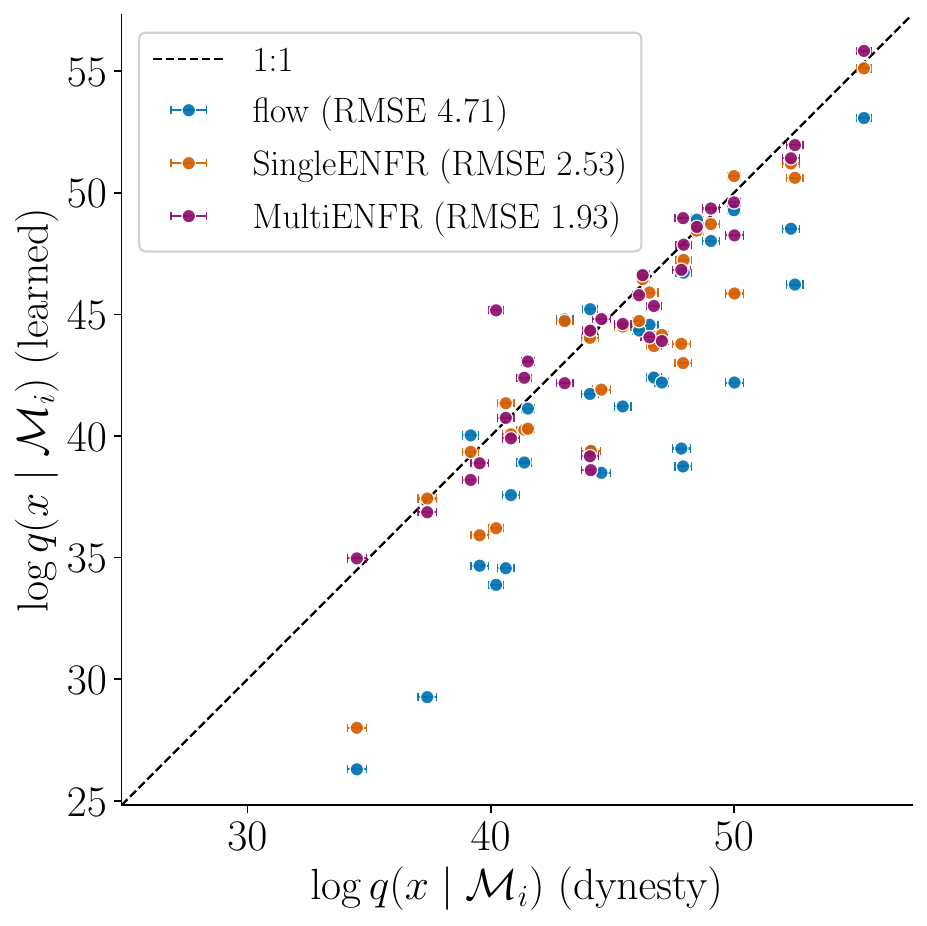}
\caption{}
\label{fig:ode_true_pred}
\end{subfigure}
\hfill
\begin{subfigure}[b]{.67\textwidth}
    \includegraphics[width=\textwidth]{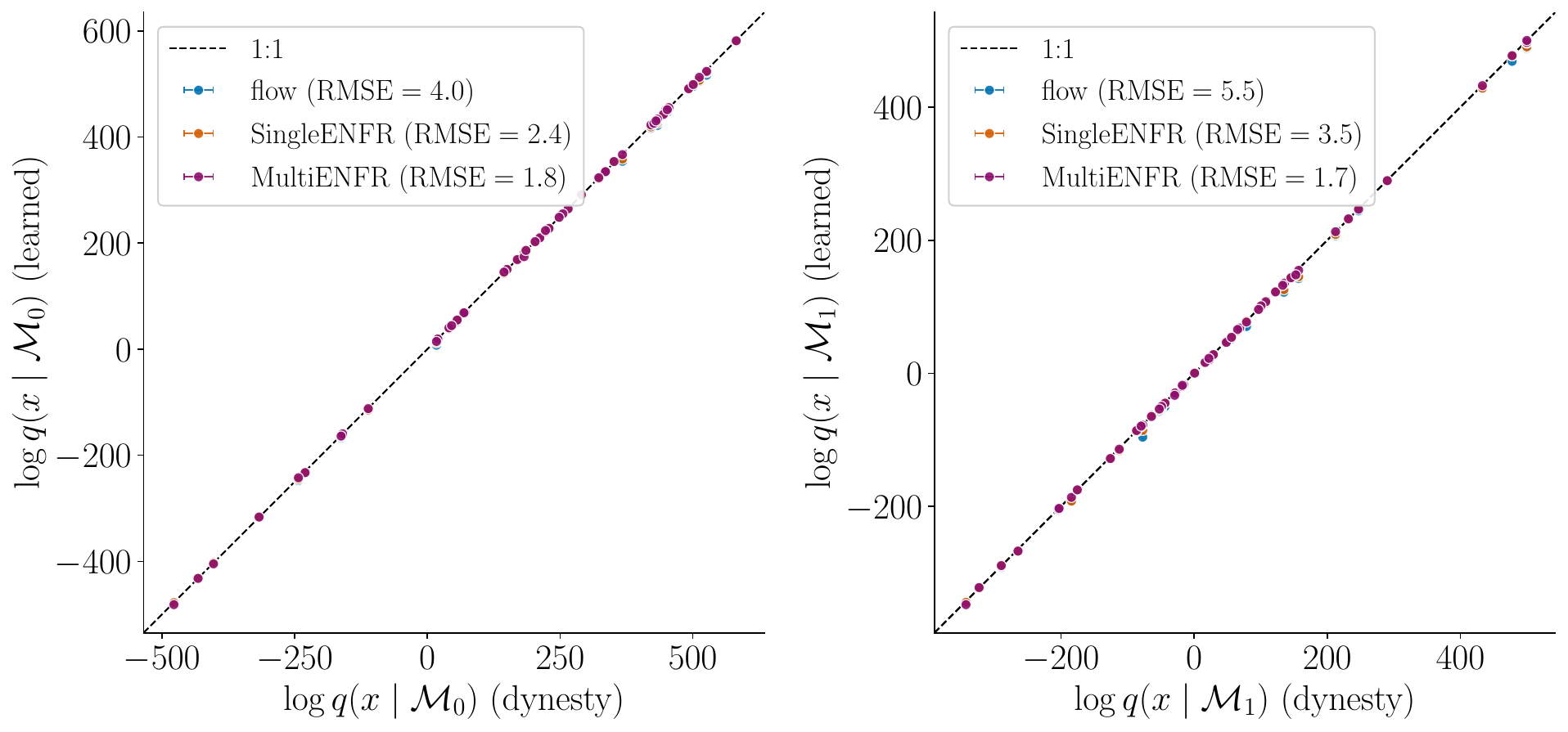}
\caption{}
\label{fig:ddm_pred_true}
\end{subfigure}
\end{center}
\caption{(a) Evidence estimation for noisy observations of an ODE driven time series using flow matching and our ENFR models. MultiENFR outperforms the other models. (b) Diffusion drift application of standard flow matching and our ENFR networks. MultiENFR shows the best performance and good evidence estimation over a large dynamical range.}
\end{figure}

\subsection{Simulation-based applications}

In this section we present examples where the true likelihood is not available, and therefore no likelihood-based estimate of the evidence can be compared to. As we have no way of assessing the evidence estimates in this case, we instead perform blind coverage testing to test the resulting Bayes factors~\citep{zadrozny_02,niculescu_05,jeffrey_24}. Given the evidence estimate, we compute the model posteriors $p_i = p(\mathcal{M}_i|x)$. For a given model $\mathcal{M}_i$, we can then bin the model posterior predictions for a sample of test simulations in probability bins. Among simulations assigned posterior probability $p_i$ to model $\mathcal M_i$, approximately a fraction $p_i$ should actually have been generated from that model. As we know the underlying ground truth model to each simulation we can perform such a test.

\subsubsection{The Hodgkin-Huxley model}

The Hodgkin-Huxley model is a biophysical model for the behavior of biological neurons. The neuron membrane is modeled as an electrical circuit and the observed quantity is the membrane voltage $V(t)$, whose dynamics are given by an SDE:

\begin{equation}
C_m dV = \big(I_{\mathrm{ext}}(t) - I_{\mathrm{Na}} - I_{\mathrm K} - I_{\mathrm L} - z_0 I_{\mathrm M} - z_1 I_{\mathrm{CaL}}\big) dt + \sigma dW,
\end{equation}   

where $C_m$ is a capacitance, and $I$ are currents induced by present elements. Each $I$ except for $I_L$ and $I_\mathrm{ext}$ are described by so-called gating variables which are defined by a system of coupled ODEs driven by the voltage, forming one coupled system. We consider four different models $\mathcal{M}_{z_0z_1}$, the first of which only contains the sodium (Na), potassium (K) and leak (L) contributions ($\mathcal{M}_{00}$). The other three models are extensions adding the slow potassium (M) ($\mathcal{M}_{10}$), and L-type calcium (CaL) ($\mathcal{M}_{01}$) contributions, or both ($\mathcal{M}_{11}$).  We give more details in Appendix~\ref{app:hodgkin}. Figure~\ref{fig:hodgkin} shows the blind calibration test for our MultiENFR model. We find that the resulting Bayes factors are well calibrated.

\begin{figure}[h]
\centering

\makebox[\textwidth][c]{%
\begin{subfigure}[b]{.35\textwidth}
    \centering
    \includegraphics[width=\textwidth]{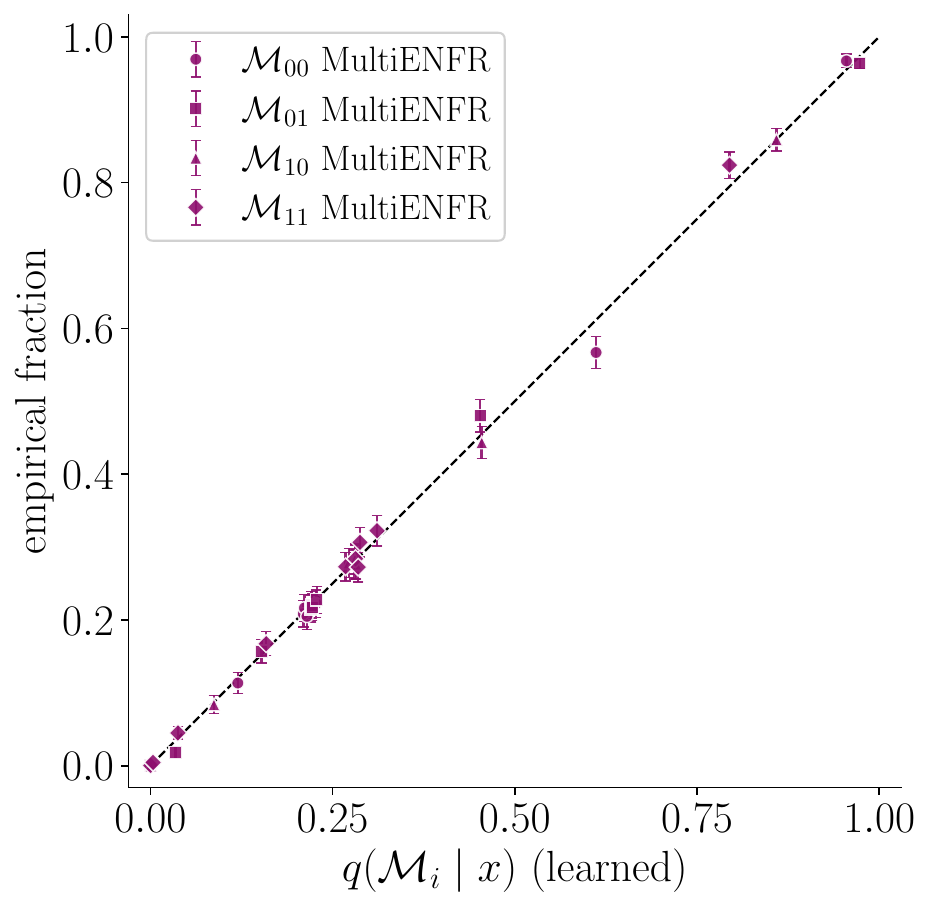}
    \caption{}
    \label{fig:hodgkin}
\end{subfigure}
\hspace{0.08\textwidth}
\begin{subfigure}[b]{.35\textwidth}
    \centering
    \includegraphics[width=\textwidth]{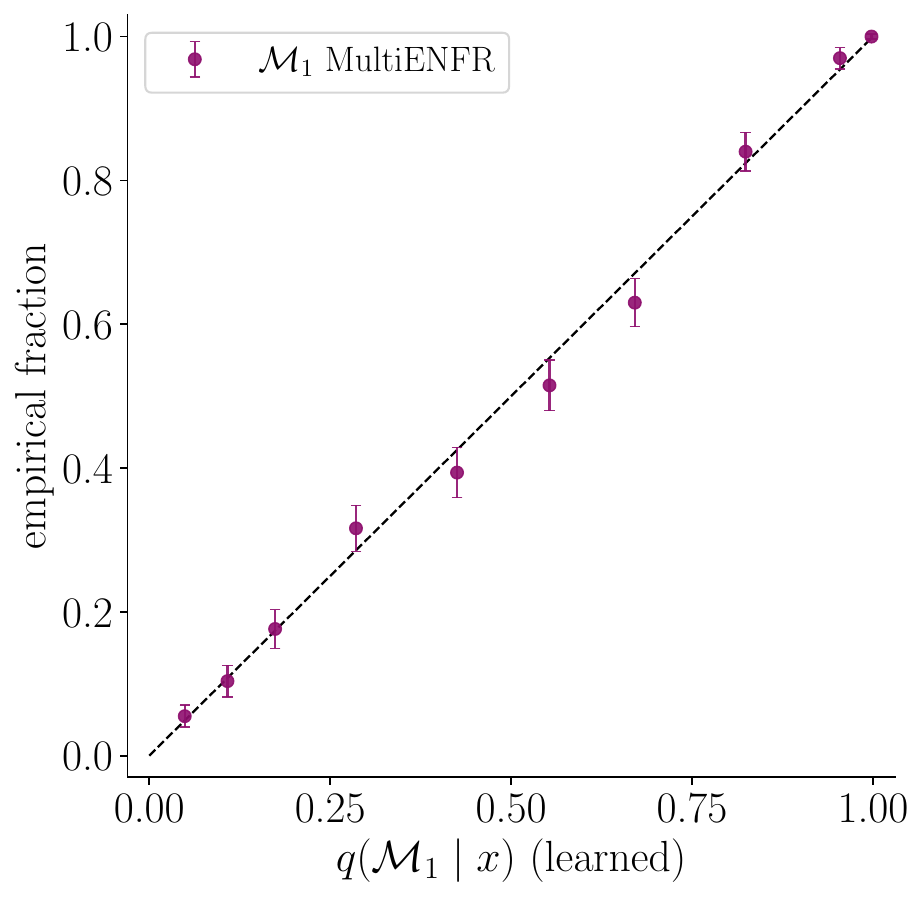}
    \caption{}
    \label{fig:bias}
\end{subfigure}%
}

\caption{Blind coverage testing for the simulation-based applications for (a) the Hodgkin-Huxley application, and (b) the Galaxy assembly bias application. As we do not have access to a likelihood-based evidence estimate for our simulation based models we test whether the resulting Bayes factors are well calibrated. 
We find that our MultiENFR architecture is well calibrated in both cases. The error bars show $2\sigma$ binomial errors.}
\end{figure}

\subsubsection{Galaxy assembly bias}

In cosmology, we observe the 3D distribution of galaxies in the Universe as a biased tracer of the underlying dark matter density field. However, theoretical predictions can only accurately model the dark matter distribution. It is therefore of major interest to model the connection between dark matter and galaxies, the so-called bias, and understand what are its major drivers. 
We compare two bias models using $362$-dimensional summary statistics from simulated galaxy catalogs: a mass-only model and an extension with assembly and phase-space bias.
Details and parameter priors are given in Appendix~\ref{app:bias}.
Figure~\ref{fig:bias} shows that the MultiENFR model is well calibrated and could be used to determine the underlying bias model from observations in a likelihood-free way.

\section{Fast evidences in action: hypothesis testing}\label{sec:hypothesis}

\begin{wrapfigure}{r}{0.5\textwidth}
    \vspace{-1.2\baselineskip}   
    \centering\includegraphics[width=\linewidth]{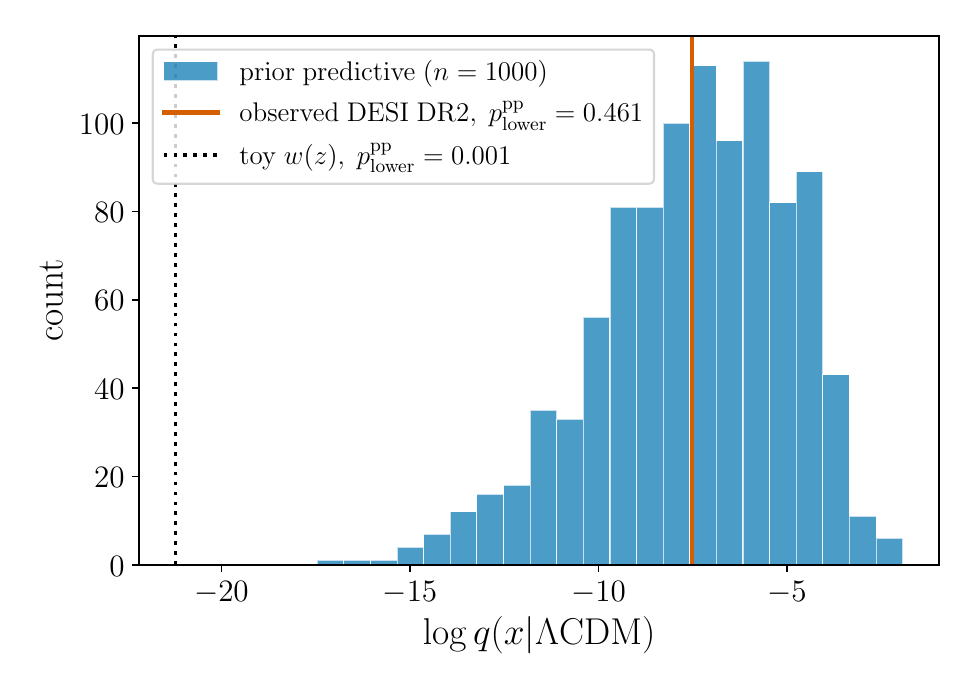}
    \caption{Evidence estimates of the MultiENFR model on prior predictive simulations and the DESI DR2 data under the $\Lambda$CDM model. While 46.1\% of simulations have lower evidence than the data, no simulation has lower evidence than the toy $w(z)$ model, and we can clearly pick up on the model misspecification. }
    \label{fig:p_value}
    \vspace{-\baselineskip}      
\end{wrapfigure}

As all of the models presented here can estimate evidences extremely fast once they are trained, we can use the evidence in traditionally intractable ways. For example, we can turn the evidence into a test statistic and run a hypothesis test~\citep{box_80}. We conduct such an analysis by drawing 1000 samples from the prior predictive of our baryonic acoustic oscillations application, and use our trained MultiENFR model to estimate the evidence of each simulation under the $\Lambda$CDM model. We consider this our null hypothesis. Comparing this to the evidence estimate of the data, allows us to quantify a prior predictive p-value: 
\begin{equation}
    p^{pp}_\mathrm{lower} = P[q_{\phi,\psi}(x_{pp}|\mathcal{M})\leq q_{\phi,\psi}(x_\mathrm{obs}|\mathcal{M})],
\end{equation}
i.e. we can estimate how many simulations have lower evidence than the data under a given model. We use the finite sample estimator $\hat{p}^{pp}_\mathrm{lower}= \frac{R+1}{n+1}$, where $R$ is the number of simulations with lower or equal evidence.

Figure~\ref{fig:p_value} shows the evidence estimates of 1000 test simulations from the prior predictive as well as the evidence estimate of the DESI DR2 data for the $\Lambda$CDM model. We infer $p^{pp}_\mathrm{lower}=0.461$ finding no indication of model misspecification for our dataset. We also show a simulation with a very non-standard dark energy parameterization (see Appendix~\ref{app:toy_w}) and could reject the null hypothesis at $p^{pp}_\mathrm{lower}=0.001$.

\section{Limitations and Conclusions}

Two main limitations remain. First, in our GMM scaling test the flow reference degrades beyond $\sim 500$ dimensions and the evidence network recovers only a small correction. High-dimensional density estimation remains a bottleneck. Second, our benchmarks use $\mathcal{O}(10^5)$ simulations per model, which many scientific simulators cannot afford. Encouragingly, the evidence-network correction is most effective precisely in the small data regime (Figure~\ref{fig:scaling}b), so the method is well placed for the simulation-limited settings where it is needed most.

Within these limits, ENFR provides a practical framework for simulation-based estimation of the Bayesian evidence by combining a flow-matched reference with an evidence network. The single-model version (SingleENFR) improves on the flow it is built on in every benchmark, and the multi-model version (MultiENFR), based on a new pairwise l-POP loss for which we prove asymptotic recovery of all Bayes factors, is the most accurate estimator across our scientific applications. Because inference is amortized, the evidence becomes cheap enough to use as a test statistic, which we demonstrated with a prior predictive misspecification test on DESI DR2 data. Our work pushes the simulation-based paradigm beyond parameter estimation, unlocking more of the standard Bayesian toolkit for the simulation-based case: principled model misspecification testing, goodness-of-fit assessment, and model comparison through the Bayesian evidence.

\subsection*{AI use statement}

Generative AI tools were used throughout the preparation of this manuscript. We used Claude Opus 5 and GPT-5.4 and GPT-5.5 for setting up and running numerical experiments and analyses, and discussing the methodology of this project. This includes generation of code which produced the training datasets used in this work. Generated code for our core methods has been checked by the authors. We have also used generative AI to assist in formatting the mathematical proofs presented in this work.  GPT-5.5 and GPT-5.6 Sol were used for reviewing this manuscript and stylistic editing. We have not used generative AI tools to draft entire sections of this manuscript. We have not used generative AI tools for the direct generation of any plot. The responsibility stays with the authors.

\subsection*{Ethics statement}

Our work is focused on the evaluation of simulation-based models for scientific reasons.

\subsection*{Reproducibility statement}

We provide the details necessary for reproducing the results presented here. We present in great detail the setups of our experiments along with preprocessing steps in Appendix~\ref{app:experiments}. The details about the architectures used are in Appendix~\ref{app:architectures}.

\subsubsection*{Acknowledgments}

We would like to thank Adrian Bayer, Konstantin Leyde, Henrique Rubira, and Nils Sch\"oneberg for stimulating and insightful discussions. KL acknowledges support via the KISS consortium (05D23WM1) funded by the German Federal Ministry of Education and Research BMBF in the ErUM-Data action plan. We are also grateful for support from the Cambridge-LMU strategic partnership.
CCL, FVN, and SK would like to thank the Aspen Center for Physics, which is supported by National Science Foundation grant PHY-2210452, where part of this work was discussed.
S.K. has been partially
supported by STFC consolidated grants ST/T000694/1 and ST/X000664/1.
The authors acknowledge support from the Simons Foundation as part of the Simons Collaboration
on Learning the Universe.
Part of this work was done at the Flatiron Institute as part of the CCA predoctoral program. The Center for Computational Astrophysics is supported by the Simons Foundation.
This research used data obtained with the Dark Energy Spectroscopic Instrument (DESI).
DESI construction and operations is managed by the Lawrence Berkeley National Laboratory.
This material is based upon work supported by the U.S. Department of Energy, Office of Science,
Office of High-Energy Physics, under Contract No. DE--AC02--05CH11231, and by the National
Energy Research Scientific Computing Center, a DOE Office of Science User Facility under the
same contract.
Additional support for DESI was provided by the U.S. National Science Foundation (NSF),
Division of Astronomical Sciences under Contract No. AST-0950945 to the NSF's National
Optical-Infrared Astronomy Research Laboratory; the Science and Technology Facilities Council
of the United Kingdom; the Gordon and Betty Moore Foundation; the Heising-Simons Foundation;
the French Alternative Energies and Atomic Energy Commission (CEA); the National Council of
Humanities, Science and Technology of Mexico (CONAHCYT); the Ministry of Science and Innovation
of Spain (MICINN), and by the DESI Member Institutions.
The DESI collaboration is honored to be permitted to conduct scientific research on I'oligam
Du'ag (Kitt Peak), a mountain with particular significance to the Tohono O'odham Nation.
Any opinions, findings, and conclusions or recommendations expressed in this material are those
of the author(s) and do not necessarily reflect the views of the U.S. National Science Foundation,
the U.S. Department of Energy, or any of the listed funding agencies.

\bibliography{bibliography}
\bibliographystyle{iclr2027_conference}

\appendix

\section{Details about the experiments} \label{app:experiments}

Here we provide more details about the numerical experiments we ran.

\subsection{GMM} \label{app:GMM}

Our GMM is defined as follows:

\begin{equation}
p(x)=\sum_{k=1}^{3} w_k\,\mathcal{N}(x\mid \mu_k,\Sigma_k)
\qquad
w=(0.45,\,0.35,\,0.20).
\end{equation}

The components are placed on a sphere of radius \(r=3.8\), i.e. \(\mu_k=r\,u_k\) with \(u_k\) sampled uniformly from the unit sphere. Each covariance is constructed as 

\begin{equation}
    \Sigma_k = D_k + \beta\,L_kL_k^{\mathsf T} + \epsilon I,
\end{equation}  

where \(D_k\) is a diagonal matrix with entries \(d_{k,i}\sim\mathcal{U}[0.7,1.8]\), \(L_k\sim\mathcal{N}(0,1/\sqrt{3})\in\mathbb{R}^{d\times 3}\) is a random factor, \(\beta=0.55\) controls correlation strength and is fixed in our analysis, and \(\epsilon=0.15\) is a jitter term ensuring positive definiteness and stable inverses. 

For the dimensionality scaling we draw 5M samples. 4.5M are used for training and 500k for the evaluation set for the flow. EN and CTSM use 4M for training and 1M for validation for the dimensionality scaling. We test on 200 more fresh draws from the GMM, which are not used in training or evaluation. We preprocess the data by standardizing it with parameters estimated from the training set:

\begin{equation}\label{equ:standardize}
    x = \frac{x_\mathrm{raw} - \mu_\mathrm{train}}{\sigma_\mathrm{train}},
\end{equation}

where $\mu_\mathrm{train}$ is the estimated mean of the training set and $\sigma_\mathrm{train}$ is the estimated standard deviation of the training set. As we end up estimating densities, we have to take into account the Jacobian $J$ of this transformation:

\begin{equation}
    \log p(x_\mathrm{raw}) = \log p(x) + \log \left| \det J \right|.
\end{equation}

The entire runtime for the dimensionality scaling of all models for one of the seeds presented here was $\sim 55$ hours in wall clock, where we trained on a single NVIDIA A40 48GB VRAM.

\subsection{BAO and priors} \label{app:BAO}

Baryonic acoustic oscillations are a characteristic scale in the Universe that is set by traveling sound waves in the early Universe. At some point, due to temperature changes, these sound waves can no longer propagate and are frozen in. This imprints the sound horizon scale which is the distance of how far a sound wave could have traveled until it is frozen. Today, we can still observe this scale by measuring the statistics of the galaxies that have formed.

The observed BAO scale depends on the expansion history of the Universe. Along the line of sight, it is related to the Hubble distance,

\begin{equation}
D_H(z) = \frac{c}{H(z)},
\end{equation}

while perpendicular to the line of sight it is related to the transverse comoving distance,

\begin{equation}
D_M(z)=
\begin{cases}
\dfrac{c}{H_0\sqrt{\Omega_k}}
\sinh\!\left[
\sqrt{\Omega_k}\,\dfrac{H_0 D_C(z)}{c}
\right],
& \Omega_k>0, \\[1.2ex]
D_C(z),
& \Omega_k=0, \\[1.2ex]
\dfrac{c}{H_0\sqrt{|\Omega_k|}}
\sin\!\left[
\sqrt{|\Omega_k|}\,\dfrac{H_0 D_C(z)}{c}
\right],
& \Omega_k<0,
\end{cases}
\end{equation}

where
\begin{equation}
D_C(z)=c\int_0^z \frac{\mathrm{d}z'}{H(z')}.
\end{equation}

Here, \(c\) is the speed of light, \(z\) is the redshift, and \(H(z)\) is the Hubble expansion rate at redshift \(z\). When the BAO signal is measured isotropically, these two distances are combined into the volume-averaged distance,

\begin{equation}
D_V(z) =
\left[
z\,D_H(z)\,D_M^2(z)
\right]^{1/3}.
\end{equation}

BAO measurements therefore constrain the dimensionless combinations \(D_H(z)/r_d\), \(D_M(z)/r_d\), or \(D_V(z)/r_d\), depending on the analysis, where \(r_d\) is the sound horizon at the baryon drag epoch. The data vector consists of 13 of these quantities measured at 7 different redshifts.

Table~\ref{tab:bao_priors} shows the boundaries of our uniform priors on the parameters of the cosmological models. We follow the choices by~\cite{ong_26}. We run simulations using the Gaussian likelihood and the DESI DR2 covariance~\cite{desi_25} using the \texttt{cobaya}~\citep{torrado_21} code. We run 500k simulations of each model for training and 2k for evaluation. For each model we run nested sampling for an additional 100 simulations per model.  Due to the presence of outliers in the data, before applying Equation~\ref{equ:standardize}, we apply the transform:

\begin{equation}\label{equ:asinh}
    x = \mathrm{asinh}\left(\frac{x_\mathrm{raw}-\mathrm{median}(x_\mathrm{raw})}{\mathrm{IQR}}\right),
\end{equation}

where $\mathrm{IQR}$ is the range between the 25th and 75th percentile. We subsequently standardize according to Equation~\ref{equ:standardize}. For density estimation, the Jacobian of this transformation is taken into account in our experiments.

We ran the BAO experiments on a single NVIDIA A40 with 48GB VRAM.

\begin{table}
\centering
\caption{Boundaries of the uniform priors for the parameters of the five cosmological models used in our BAO example. The six base parameters of the $\Lambda$CDM model are shared by every model. Each extension introduces either one or two more parameters. }
\label{tab:bao_priors}
\begin{tabular}{lccccc}
\toprule
Parameter & $\Lambda$CDM & $w$CDM & $w_0w_a$CDM & $\nu$CDM & $\Omega_k$CDM \\
\midrule
\multicolumn{6}{l}{Base parameters (shared)} \\
$H_0$ & \multicolumn{5}{c}{$[20,\,100]$} \\
$\Omega_b h^2$ & \multicolumn{5}{c}{$[0.005,\,0.1]$} \\
$\Omega_c h^2$ & \multicolumn{5}{c}{$[0.001,\,0.99]$} \\
$\tau$ & \multicolumn{5}{c}{$[0.01,\,0.8]$} \\
$\log(10^{10}A_s)$ & \multicolumn{5}{c}{$[1.61,\,3.91]$} \\
$n_s$ & \multicolumn{5}{c}{$[0.8,\,1.2]$} \\
\midrule
\midrule
\multicolumn{6}{l}{Extension parameters} \\
$w$ (or $w_0$) & \textemdash & $[-3,\,-0.333]$ & $[-3,\,1]$ & \textemdash & \textemdash \\
$w_a$ & \textemdash & \textemdash & $[-3,\,2]$ & \textemdash & \textemdash \\
$\sum m_\nu$ [eV] & \textemdash & \textemdash & \textemdash & $[0.06,\,2]$ & \textemdash \\
$\Omega_k$ & \textemdash & \textemdash & \textemdash & \textemdash & $[-0.3,\,0.3]$ \\
\midrule
Number of Parameters & 6 & 7 & 8 & 7 & 7 \\
\bottomrule
\end{tabular}
\end{table}

\subsection{Toy dark energy model} \label{app:toy_w}

As a toy model for late-time model misspecification, we consider a smooth transition in the dark-energy equation of state,
\begin{equation}
    w(z)
    =
    -1
    +
    \frac{\Delta w}{2}
    \left[
        1
        -
        \tanh\left(
            \frac{z-z_t}{\Delta z}
        \right)
    \right].
\end{equation}
We choose
\begin{equation}
    \Delta w = 0.4,
    \qquad
    z_t = 0.6,
    \qquad
    \Delta z = 0.15.
\end{equation}
The model therefore approaches $w=-1$ at high redshift, while at low redshift it transitions to $w\simeq -0.6$. 

\subsection{ODE application details and priors} \label{app:ode}

Table~\ref{tab:ode_priors} shows the priors on the parameters. We generate a training set of 1M simulations per model, 10\% of which is used as validation data. We apply the preprocessing transformation from Equation~\ref{equ:standardize}. We quote our results in Table~\ref{tab:ode_rmse}.

\begin{table}[t]
\centering
\caption{Prior distributions for the ODE oscillator models. All priors are independent uniforms;
parameters shared between models have identical priors.}
\label{tab:ode_priors}
\begin{tabular}{lcccc}
\toprule
Parameter & Prior & Harmonic & Damped & Duffing \\
\midrule
$x_0$      & $\mathcal{U}(-2,\,2)$     & \checkmark & \checkmark & \checkmark \\
$v_0$      & $\mathcal{U}(-2,\,2)$     & \checkmark & \checkmark & \checkmark \\
$\omega$   & $\mathcal{U}(0.4,\,2.5)$  & \checkmark & \checkmark & \checkmark \\
$\gamma$   & $\mathcal{U}(0,\,0.4)$    &            & \checkmark & \checkmark \\
$\beta$    & $\mathcal{U}(-0.8,\,0.8)$ &            &            & \checkmark \\
\bottomrule
\end{tabular}
\end{table}

\begin{table}
  \centering
  \caption{RMSE of the learned evidence against nested sampling (dynesty) on the ODE
    oscillator benchmark, over 10 held-out simulations per model.}
  \label{tab:ode_rmse}
  \begin{tabular}{lrrrr}
    \hline
    Method & harmonic & damped & Duffing & All \\
    \hline
    flow & 6.13 & 3.69 & 3.91 & 4.71 \\
    SingleENFR & 2.54 & 2.44 & 2.60 & 2.53 \\
    MultiENFR & 2.36 & 0.94 & 2.17 & 1.93 \\
    flow + CTSM-SB (ens.) & 135.47 & 26.23 & 7.96 & 79.80 \\
    \hline
  \end{tabular}
\end{table}

\subsection{DDM} \label{app:ddm}

We show the priors of our DDM experiment in Table~\ref{tab:ddm_priors},
adopting the parameter prior ranges of
\citet{schroder_24}. The likelihood is available by solving the corresponding Fokker-Planck equation. We use the pyDDM implementation for this and refer to its documentation for further details~\citep{pyddm_20}. We generate 500k simulations for each model for training out of which we split 10\% for validation. We apply the preprocessing transformation from Equation~\ref{equ:asinh}. We show our results in Table~\ref{tab:ddm_density}.

\begin{table}[t]
\centering
\caption{Prior distributions for the drift-diffusion models.}
\label{tab:ddm_priors}
\begin{tabular}{llccc}
\toprule
Parameter & Description & Prior & $\mathcal{M}_0$ & $\mathcal{M}_1$ \\
\midrule
$v$            & drift rate                 & $\mathcal{U}(0,\,5)$       & \checkmark & \checkmark \\
$\lambda$      & leak rate                  & $\mathcal{U}(-20,\,-5)$    &            & \checkmark \\
$B$            & initial boundary height    & $\mathcal{U}(0.3,\,2)$     & \checkmark & \checkmark \\
$\kappa$       & boundary collapse rate     & $\mathcal{U}(0.5,\,1.5)$   & \checkmark & \checkmark \\
$t_{\mathrm{nd}}$ & non-decision time [s]   & $\mathcal{U}(0.1,\,0.3)$   & \checkmark & \checkmark \\
\bottomrule
\end{tabular}
\end{table}

\begin{table}
  \centering
  \caption{Estimated evidence against nested sampling (dynesty) on the DDM benchmark, over 50 held-out datasets from each model.}
  \label{tab:ddm_density}
  \begin{tabular}{lrr}
    \toprule
    Model & RMSE under $\mathcal{M}_0$ & RMSE under $\mathcal{M}_1$ \\
    \midrule
    flow & 4.0 &  5.5 \\
    SingleENFR & 2.4 & 3.5  \\
    MultiENFR & 1.8 & 1.7 \\
    flow + CTSM-SB & 20.3  & 115.4  \\
    \bottomrule
  \end{tabular}
\end{table}

\subsection{Hodgkin-Huxley} \label{app:hodgkin}

The currents are defined as follows:

\begin{equation}
\begin{aligned}
I_{\mathrm{Na}} &= g_{\mathrm{Na}} m^3 h(V - E_{\mathrm{Na}})\\
I_{\mathrm K} &= g_{\mathrm K} n^4 (V - E_{\mathrm K})\\
I_{\mathrm L} &= g_L (V - E_L)\\
I_{\mathrm M} &= g_M p(V - E_{\mathrm K})\\ 
I_{\mathrm{CaL}} &= g_{\mathrm{Ca}} q^2 r(V - E_{\mathrm{Ca}}),\\
\end{aligned}
\end{equation}

where $g$ are the maximal conductances, $E$ are reversal potentials and $m,h,n,p,q,r \in [0,1]$ are gating variables. The gating variables $x_{\mathrm{gate},a} \in \{m,h,n\}$ are obtained by a first order differential equation: $\dot x_{\mathrm{gate},a} = \alpha_x(V-V_t)(1-x_{\mathrm{gate},a}) - \beta_x(V-V_t)\,x_{\mathrm{gate},a}$, where $\alpha$ and $\beta$ are voltage-dependent rate functions as defined in~\cite{posposchil_08}. The gating variables $x_{\mathrm{gate},b} \in \{q,r\}$ evolve according to $\dot x_{\mathrm{gate}_b} = \alpha_x(V)(1-x_{\mathrm{gate}_b}) - \beta_x(V)\,x_{\mathrm{gate},b}$. p is obtained through $\dot p = (p_\infty(V) - p)/\tau_p(V)$ where $p_\infty$ and $\tau_p$ are defined by:

\begin{equation}
\begin{aligned}
p_\infty &= 1/(1+e^{-(V+35)/10}) \\
\tau_p &= \tau_M / (3.3e^{(V+35)/20} + e^{-(V+35)/20}). \\
\end{aligned} 
\end{equation}

All gates are initialized at $V(0) = -70$\,mV. The
injected current is a step of $2\,\mu\mathrm{A\,cm^{-2}}$ between 225 and
1225\,ms with added noise, $I_{\mathrm{inj}}(t_k) = I_{\mathrm{step}}(t_k)
+ \sigma\,\xi_k/\sqrt{\Delta t}$ where $\xi_k \sim \mathcal N(0,1)$. We integrate the voltage using a Crank--Nicolson integrator at $\Delta t = 0.025$ ms.

From each simulation we extract 13 hand-picked summaries: resting
potential, spike latency, inverse first/mean/median inter-spike interval,
two ISI-adaptation indices, two AP-amplitude adaptation indices, AP amplitude,
inverse AP width, and the coefficient of variation and Fano factor of the AP
amplitudes. Some summaries are undefined in some of the simulations, e.g. due to the absence of spikes.
In that case they are set to
their physical limit (latency 1000\,ms, inverse ISI and
inverse width 0, adaptation indices 1, amplitude, CV and Fano factor 0).

We use the \texttt{jaxley}~\cite{jaxley_25} library for implementation. We generate 400k simulations of each model for training and use 10\% of that for validation. 
We adopt the parameter prior ranges of~\citet{schroder_24}, which were originally
adapted from~\citet{goncalves_20}. We present those prior choices in Table~\ref{tab:hh_priors} and list fixed parameters in Table~\ref{tab:hh_fixed}. We preprocess using Equation~\ref{equ:standardize}.

\begin{table}[t]
\centering
\caption{Prior distributions for the Hodgkin--Huxley models.}
\label{tab:hh_priors}
\begin{tabular}{lll}
\toprule
Parameter  & Prior & Unit  \\
\midrule
$g_L$                         & $\mathcal{U}(10^{-6},\,3\times10^{-4})$            & $\mathrm{S\,cm^{-2}}$ \\
$E_L$                  & $\mathcal{U}(-80,\,-60)$                          & $\mathrm{mV}$         \\
$g_{\mathrm{K}}$   & $\mathcal{U}(1.5\times10^{-3},\,1.5\times10^{-2})$ & $\mathrm{S\,cm^{-2}}$ \\
$V_t$              & $\mathcal{U}(-70,\,-50)$                  & $\mathrm{mV}$         \\
$g_{\mathrm{Na}}$           & $\mathcal{U}(8\times10^{-3},\,8\times10^{-2})$    & $\mathrm{S\,cm^{-2}}$ \\
$\sigma$           & $\mathcal{U}(10^{-4},\,0.15)$                   & $\mu\mathrm{A\,cm^{-2}\,ms^{1/2}}$ \\
\midrule
$g_{\mathrm{M}}$         & $\mathcal{U}(10^{-5},\,6\times10^{-4})$           & $\mathrm{S\,cm^{-2}}$            \\
$\tau_{\mathrm{M}}$  & $\mathcal{U}(200,\,2000)$                     & $\mathrm{ms}$                    \\
\midrule
$g_{\mathrm{Ca}}$  & $\mathcal{U}(5\times10^{-5},\,10^{-3})$          & $\mathrm{S\,cm^{-2}}$            \\
\bottomrule
\end{tabular}
\end{table}

\begin{table}[t]
\centering
\caption{Fixed parameters for the Hodgkin--Huxley simulations.}
\label{tab:hh_fixed}
\begin{tabular}{lll}
\toprule
Quantity & Value & Unit \\
\midrule
$E_{\mathrm{Na}}$  & $50$    & $\mathrm{mV}$ \\
$E_{\mathrm{K}}$    & $-90$   & $\mathrm{mV}$ \\
$E_{\mathrm{Ca}}$ & $120$ & $\mathrm{mV}$ \\
$C_m$                   & $1$     & $\mu\mathrm{F\,cm^{-2}}$ \\
Compartment length / radius                  & $10$ / $1$ & $\mu\mathrm{m}$ \\
Step current density (225--1225\,ms)         & $2$     & $\mu\mathrm{A\,cm^{-2}}$ \\
Trace duration / integration step            & $1450$ / $0.025$ & $\mathrm{ms}$ \\
\bottomrule
\end{tabular}
\end{table}

\subsection{Galaxy assembly bias} \label{app:bias}

We populate halo catalogs from 2000 high-resolution Quijote
Latin-hypercube simulations~\citep{gb_quijote2020} at $z=0.5$ in periodic cubes of side length
$1000\,h^{-1}\mathrm{Mpc}$. Model~$\mathcal{M}_0$ uses a mass-only
halo occupation distribution (HOD)~\citep{gb_zheng2007}, implemented with
Halotools~\citep{gb_halotools2017}, and baseline mean occupations
\begin{equation}
\begin{aligned}
\mu_{\rm cen}(M)
  &= \frac{1}{2}\left[1+\operatorname{erf}\left(
       \frac{\log_{10}M-\log_{10}M_{\min}}{\sigma_{\log M}}
       \right)\right],\\
\mu_{\rm sat}(M)
  &= \mu_{\rm cen}(M)\left(\frac{M-M_0}{M_1}\right)^\alpha,
  \qquad M>M_0,
\end{aligned}
\end{equation}

with $\mu_{\rm sat}=0$ otherwise. Central and satellite counts are drawn
separately from Bernoulli and Poisson distributions. Centrals follow the
halo centers and bulk velocities while satellites have Navarro--Frenk--White
radial profiles and Jeans-equilibrium velocities.
Model~$\mathcal{M}_1$ adds concentration-dependent occupation at fixed
mass using a decorated HOD~\citep{gb_decorated2016}, with strengths
$A_{\rm cen}$ and $A_{\rm sat}$. It also includes velocity bias
\citep[e.g.][]{gb_velocity2015} and a rescaled satellite concentration:
central velocity dispersion $\alpha_{\rm c}V_{\rm vir}$, satellite
velocity multiplier $\alpha_{\rm s}$ relative to the halo, and satellite
concentration $c_{\rm gal}=F_{\rm sat}c_{\rm halo}$. For more details on the 
HOD models see~\cite{gb_zheng2007,gb_decorated2016}.

Table~\ref{tab:galaxy_bias_priors} lists the priors in our analysis. Shared parameters
have identical priors and the models have equal prior probability.
For $\log_{10}M_{\min}$, we uniformly sample the $z=0.5$ entries of
$10^6$ stored draws from an empirical prior separately constrained using
the redshift distribution of the BOSS DR12 CMASS North Galactic Cap
sample~\citep{gb_boss2017}. Other HOD parameters
are drawn independently as tabulated. No HOD redshift evolution is used.

\begin{table}[t]
\centering
\small
\caption{Priors for the Quijote galaxy-bias comparison.
$\mathcal{U}(a,b)$ denotes a uniform distribution and
$\mathcal{TN}(0,0.2^2;[-1,1])$ a zero-mean Gaussian of standard deviation
$0.2$, truncated to $[-1,1]$. Mass logarithms are base ten with masses in
$h^{-1}M_\odot$. Numerical constants denote fixed values.
The HOD mass $M_1$ is distinct from the model label $\mathcal{M}_1$.}
\label{tab:galaxy_bias_priors}
\begin{tabular}{@{}lll@{}}
\toprule
Parameter & $\mathcal{M}_0$ & $\mathcal{M}_1$ \\
\midrule
$\Omega_{\rm m}$ & $\mathcal{U}(0.1,0.5)$ & $\mathcal{U}(0.1,0.5)$ \\
$\Omega_{\rm b}$ & $\mathcal{U}(0.03,0.07)$ & $\mathcal{U}(0.03,0.07)$ \\
$h$ & $\mathcal{U}(0.5,0.9)$ & $\mathcal{U}(0.5,0.9)$ \\
$n_{\rm s}$ & $\mathcal{U}(0.8,1.2)$ & $\mathcal{U}(0.8,1.2)$ \\
$\sigma_8$ & $\mathcal{U}(0.6,1.0)$ & $\mathcal{U}(0.6,1.0)$ \\
\midrule
$\log_{10}M_{\min}$ & NGC empirical prior & NGC empirical prior \\
$\sigma_{\log M}$ & $\mathcal{U}(0.1,0.6)$ & $\mathcal{U}(0.1,0.6)$ \\
$\log_{10}M_0$ & $\mathcal{U}(13,15)$ & $\mathcal{U}(13,15)$ \\
$\log_{10}M_1$ & $\mathcal{U}(13,15)$ & $\mathcal{U}(13,15)$ \\
$\alpha$ & $\mathcal{U}(0,1.5)$ & $\mathcal{U}(0,1.5)$ \\
\midrule
$A_{\rm cen}$ & $0$ & $\mathcal{TN}(0,0.2^2;[-1,1])$ \\
$A_{\rm sat}$ & $0$ & $\mathcal{TN}(0,0.2^2;[-1,1])$ \\
$\alpha_{\rm c}$ & $0$ & $\mathcal{U}(0,0.7)$ \\
$\alpha_{\rm s}$ & $1$ & $\mathcal{U}(0.2,2.0)$ \\
$F_{\rm sat}$ & $1$ & $\mathcal{U}(0.2,2.0)$ \\
\bottomrule
\end{tabular}
\end{table}

The LHID-disjoint
dataset split contains 64,000 training simulations (32,000 per model),
8,000 validation simulations and 8,000 frozen test simulations; the estimators
also reserve internal validation subsets from the training partition.
A single flow-matching model is fitted to the pooled training set, and 32,000
generated samples form a third class alongside the two equally sized real-model
classes. We train a five-member ensemble of three-output multiclass Evidence
Networks with the pairwise loss, using standardization from Equation~\ref{equ:standardize} for both the flow
and the networks.

\section{Architectures} \label{app:architectures}

We present the architectures of the neural networks used in this work. All networks are implemented in \texttt{PyTorch}~\citep[][]{pytorch_19}. We have also made use of the \texttt{numpy}~\citep[][]{numpy_20}, \texttt{scipy}~\citep[][]{scipy_20}, \texttt{sklearn}~\citep[][]{scikit-learn_11}, \texttt{torchdiffeq}~\citep{torchdiffeq_18}, and \texttt{matplotlib}~\citep[][]{matplotlib_07} libraries.

\paragraph{Software and data licenses} PyTorch, NumPy, SciPy, and scikit-learn are used under their BSD-3-Clause licenses; dynesty, pyddm, and torchdiffeq under the MIT license; Cobaya under LGPL-3.0-or-later with its stated additional conditions; jaxley under the Apache 2.0 licence; and Matplotlib under the Matplotlib license. The DESI data products are used under CC BY 4.0.

\subsection{Flow matching}\label{app:flow}

We parameterize the velocity field of the flow matching algorithm with a 256 unit, 4 hidden layer fully connected multilayer perceptron using SiLU activation functions. We train with uniform time sampling, using the \texttt{Adam}~\cite{kingma_14} optimizer with a batch size of 256 for the Hodgkin and DDM applications and 128 for all other applications. For the GMM and ODE models we use a learning rate of $10^{-4}$, for the BAO, DDM and Hodgkin models $10^{-3}$. For density estimation via integration of the flow, we use the \texttt{dopri5} integrator implemented in \texttt{torchdiffeq} to solve

\begin{equation}
    \log q_\phi(x|\mathcal{M}) = \log p(z) - \int_0^1 \nabla \cdot v_\phi(x_t, t)\, dt,
\end{equation}

where $x_t$ follows the learned ODE trajectory. We leave the impact of a thorough hyperparameter search and of a different optimizer to further study.

For the DDM application we use permutation equivariant DeepSets~\citep{zaheer_17}. First we apply a neural network to each dimension:

\begin{equation}
h_i = \phi(x_i, t),
\end{equation}

to the outputs of which we apply permutation invariant aggregation functions:

\begin{equation}
A = \big[\,\overline{h},\ \textstyle\max_i h_i,\ \min_i h_i,\ \mathrm{std}_i(h_i)\,\big].
\end{equation}

The velocity field is then learned from a concatenation of these aggregated features with the individual observation:

\begin{equation}
v_i = \rho\big(x_i,\ h_i,\ A,\ t\big).
\end{equation}

The neural networks $\phi$ and $\rho$ are parameterized as MLPs with 256 neurons in each of the 3 layers. The output dimension of $\phi$ is also 256.

\subsection{ENFR}

The evidence networks are fully connected neural networks. For the GMM example we use a rather large network with 1024 x 4 neurons with the SiLU activation function. For the BAO example we use a substantially smaller network with 128x2 neurons for MultiENFR and 128x3 neurons for SingleENFR with the ReLU activation function. For the ODE application we use 256x4 neurons with ReLU activation. For the Hodgkin application we use 128x2 neurons with ReLU activation. In all examples except for the GMM we train an ensemble of 5 evidence networks with equal weight, while the GMM test only shows single networks. Batch sizes are 128 (GMM,BAO), 256 (Hodgkin, DDM), and 1024 (ODE). 

We train using the \texttt{Adam}~\cite{kingma_14} optimizer with a learning rate of $10^{-4}$.

For the DDM experiment MultiENFR is a permutation invariant DeepSet~\citep{zaheer_17}. A first neural network delivers

\begin{equation}
    h_i = \phi(x_i),
\end{equation}

where $x_i$ are the simulated signed decision times. The output is then aggregated:

\begin{equation}
A = \big[\,\overline{h},\ \textstyle\max_i h_i,\ \min_i h_i,\ \mathrm{std}_i(h_i)\,\big],
\end{equation}

and fed into a second network

\begin{equation}
    y_j = \rho(A).
\end{equation}

$\rho$ and $\phi$ are parameterized with MLPs with two and three layers respectively, with 128 neurons in each layer, using a ReLU activation function. We train an equally weighted 5 member ensemble with learning rate $10^{-4}$ batch size 256 and patience 10 in early stopping.

\subsection{CTSM}\label{app:ctsm}

The CTSM models are also fully connected multilayer perceptrons. There are two different interpolants we use, the variance preserving bridge (VP):

\begin{equation}
    x_t =  \sqrt{1-t^2} z + t x_1 \quad \mathrm{with}\ z\sim\mathcal{N}(0,1),
\end{equation}

and the Schrödinger bridge (SB):

\begin{equation}
    x_t = (1-t) x_0 + t x_1 + \sigma \sqrt{t(1-t)} \epsilon \quad \mathrm{with}\ \epsilon\sim\mathcal{N}(0,1),
\end{equation}

where $x_0\sim q_\phi(x)$ is drawn from the flow reference and $x_1 \sim p(x|\mathcal{M})$ are data samples. We choose the noise scale $\sigma$ to be 1.

In all CTSM runs, the architectures have always been identical to the corresponding SingleENFR runs.

\section{Results of the GMM scaling experiment}

We show the results of the GMM scaling in Table~\ref{tab:GMM}. We further show an example plot of the improvement that comes from adding the evidence network in 128 dimensions in Figure~\ref{fig:GMM_true_pred}.

\begin{table}
	\caption{RMSE values of the estimated density $\log q(x)$ with respect to the true density $\log p(x)$ over 200 held out test draws (lower is better) averaged over the three seeds. The SingleENFR combination consistently outperforms the other models. The flow is shared at each seed.}
	\label{tab:GMM}
	\centering
	\begin{tabular}{lcccc} 
        \toprule
		Example  & Flow & CTSM-VP & SingleENFR & CTSM-SB + FlowRef\\
        \midrule
		GMM 64D & 0.76 & 2.14 &\bf{0.45} & 1.06 \\%
		GMM 128D & 2.30 & 7.70 & \bf{1.32} & 2.46 \\%
		GMM 256D & 6.61 & 27.87 & \bf{3.04} & 6.57 \\ %
		GMM 512D & 15.79 & 65.36 & \bf{14.44} & 15.87 \\%
		GMM 1024D & 34.99 & 163.21 & \bf{33.44} & 35.01 \\%
        \bottomrule
	\end{tabular}
\end{table}

\begin{figure}
    \centering
    \includegraphics[width=0.5\linewidth]{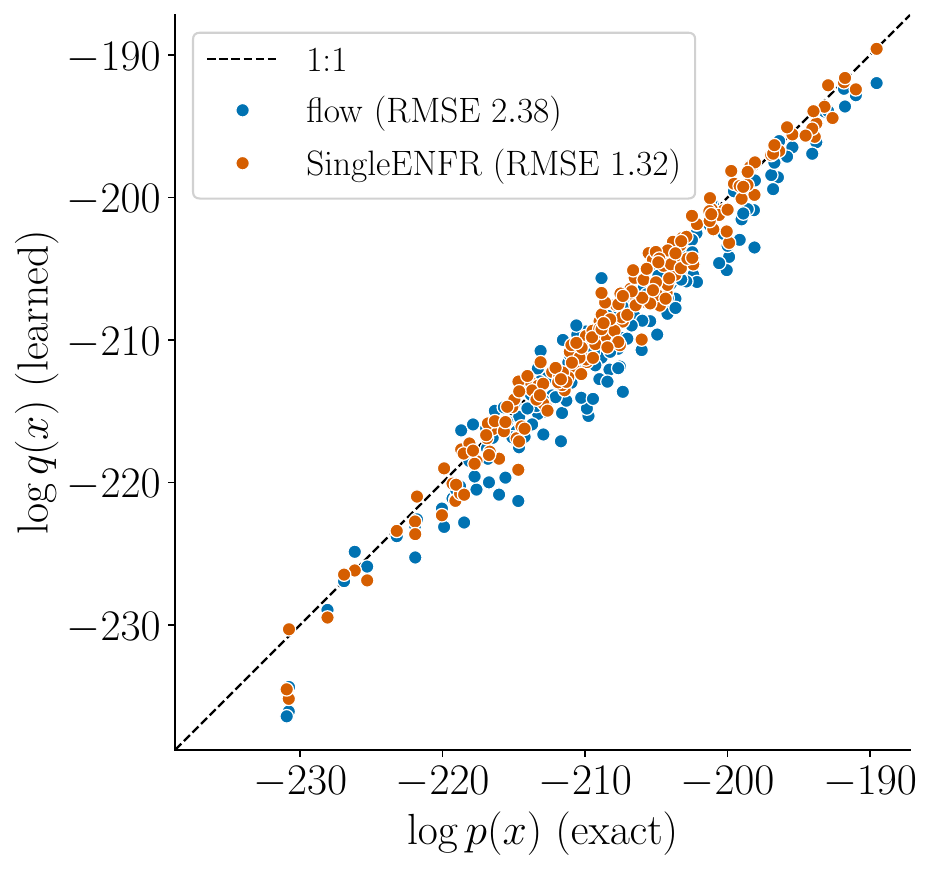}
    \caption{The performance on test points in 128 dimensions for the flow and SingleENFR.}
    \label{fig:GMM_true_pred}
\end{figure}

\section{Bayes factor analysis for DESI DR2 data}

Table~\ref{tab:bao_data_bf} shows the explicit comparison of our MultiENFR Bayes factors compared to the likelihood-based analysis.

\begin{table}
	\centering
	\caption{Bayes factors w.r.t. $\Lambda$CDM on the DESI DR2 data. The error of our simulation-based estimate is largely comparable to the error of the likelihood-based estimate in three of the four models.}
	\label{tab:bao_data_bf}
	\begin{tabular}{lccc} 
        \toprule
		Model & Likelihood Based Analysis & MultiENFR & $\Delta$ \\
		\midrule
        $w\mathrm{CDM}$ & -2.42$\pm$0.24 & -1.94 & +0.48 \\
        $w_0w_a\mathrm{CDM}$ & -1.64$\pm$0.24  & -1.40 &  +0.25 \\
        $\nu\mathrm{CDM}$ &-0.07$\pm$0.22 & 0.06 & +0.13 \\
        $\Omega_k\mathrm{CDM}$ & -1.89$\pm$0.23 & -1.65 & +0.24 \\
        \bottomrule
	\end{tabular}

\end{table}

\section{Proof of minimizer of evidence networks} \label{app:lpop}

We present the proof that the minimizer of the risk of the l-POP loss is indeed the Bayes factor if the models are weighted equally a priori.

\paragraph{Proposition.}
For a sample $(x,m)$ with $m\in\{0,1\}$, the l-POP exponential loss is
\[
\mathcal L(x,m)
=
\exp\left[
\left(\frac12-m\right)u(x)
\right],
\]
using the shorthand notation $u(x) = \mathcal{J}_\alpha(f_\psi(x))$. Assuming equal model sampling probabilities, the population optimum satisfies
\[
u^*(x)
=
\log\frac{p(x\mid \mathcal{M})}{p(x\mid \mathcal{M}_0)} = \log K,
\]
with the $^*$ superscript indicating at minimizing parameters $\psi = \hat\psi$.

\paragraph{Proof.}
Consider the population risk as a functional of \(u(\cdot)\). Under equal sampling probabilities for the two models,

\[
P(\mathcal M_0)=P(\mathcal M)=\frac12,
\]

the expected loss is

\[
\mathcal R[u]
=
\frac12
\int
\left[
p(x\mid \mathcal M_0)e^{u(x)/2}
+
p(x\mid \mathcal M)e^{-u(x)/2}
\right]dx.
\]

We proceed with standard variational calculus, differentiating with respect to some perturbation
\(\varepsilon\) at \(\varepsilon=0\):

\[
\delta\mathcal R[u;\eta]
=
\left.
\frac{d}{d\varepsilon}
\mathcal R[u+\varepsilon\eta]
\right|_{\varepsilon=0}.
\]

Thus,

\[
\delta\mathcal R[u;\eta]
=
\frac14
\int
\left[
p(x\mid \mathcal M_0)e^{u(x)/2}
-
p(x\mid \mathcal M)e^{-u(x)/2}
\right]
\eta(x)\,dx.
\]

At a minimizer \(u^*\), the variation must vanish so we find

\[
p(x\mid \mathcal M_0)e^{u^*(x)/2}
-
p(x\mid \mathcal M)e^{-u^*(x)/2}
=
0
\]

Hence,

\[
p(x\mid \mathcal M_0)e^{u^*(x)/2}
=
p(x\mid \mathcal M)e^{-u^*(x)/2}.
\]

Multiplying by \(e^{u^*(x)/2}\) gives

\[
p(x\mid \mathcal M_0)e^{u^*(x)}
=
p(x\mid \mathcal M),
\]

and therefore

\[
e^{u^*(x)}
=
\frac{p(x\mid \mathcal M)}
     {p(x\mid \mathcal M_0)}.
\]

Taking logarithms,

\[
u^*(x)
=
\log
\frac{p(x\mid \mathcal M)}
     {p(x\mid \mathcal M_0)}
=
\log K.
\]

Thus,

\[
\boxed{
u^*(x)=\log K
}.
\]

Without the shorthand this delivers

\[
\boxed{
\mathcal{J}_\alpha(f^*(x))
=
\log K
}.
\]

\section{EN correction to the flow under constant prediction} \label{app:en_const}

We use the l-POP loss EN to correct the density estimates of a generative flow. It is interesting to examine the case where the EN cannot learn anything from the data and just returns a minimizing constant, i.e. $f(x) = \mathrm{const}$. In this case we can show, that the EN correction asymptotically does not degrade the flow's evidence estimate under equal model weights.

Consider again the population risk of the l-POP loss of a constant predictor $u(x) = \mathcal{J}_\alpha(f_\psi(x)) = c$ under model priors which are reflected by the training class probabilities $p(\mathcal{M})$ and $p(\mathcal{M}_0)$:

\[
R[c]
=
\int p(\mathcal{M}_0) p(x\mid \mathcal M_0)e^{c/2}+
 p(\mathcal{M}) p(x\mid \mathcal M)e^{-c/2}\,dx
=
p(\mathcal{M}_0)e^{c/2}+p(\mathcal{M})e^{-c/2}.
\]

We consider the variation

\[
\delta\mathcal R[c;a]
=
\left.
\frac{d}{d\varepsilon}
\mathcal R[c+\varepsilon a]
\right|_{\epsilon=0},
\]

which delivers

\[
\delta\mathcal R[c;a]
=
\frac{a}{2}
\left[
p(\mathcal M_0)e^{c/2}
-
p(\mathcal M)e^{-c/2}
\right].
\]

At the variational optimum $c^*$, the variation must vanish so

\[
p(\mathcal M_0)e^{c^*/2}
-
p(\mathcal M)e^{-c^*/2}
=
0,
\]

which implies

\[
e^{c^*}
=
\frac{p(\mathcal M)}
     {p(\mathcal M_0)}.
\]

Hence the optimal constant correction is

\[
\boxed{
c^*
=
\log
\frac{p(\mathcal M)}
     {p(\mathcal M_0)}
}.
\]

For equal model weights,

\[
p(\mathcal M_0)=p(\mathcal M),
\]

and therefore

\[
\boxed{c^*=0}.
\]

The asymptotically optimal EN therefore does not degrade a well-trained flow in this case.

\section{Proof of minima of the Multiclass Evidence Network} \label{app:caten}

We present the proof that the multiclass evidence network loss function is minimized by the Bayes factors.

\paragraph{Proposition.}
Consider \(N\) models

\[
\mathcal M_1,\ldots,\mathcal M_N
\]

and define

\[
u_i(x)=\mathcal J_\alpha(f_i(x)).
\]

For a sample \(x\) generated from model \(\mathcal M_m\), consider the pairwise loss

\[
\mathcal L(x,m)
=
\frac{1}{N-1}
\sum_{j\neq m}
\exp\left[
-\frac12\left(u_m(x)-u_j(x)\right)
\right].
\]

Assume that the models are sampled with equal probability,

\[
p(\mathcal M_i)=\frac1N.
\]

Then the population-risk functional, up to the positive constant factor
\(1/[N(N-1)]\), is

\[
\mathcal R[\mathbf u]
=
\int
\sum_{i=1}^N
p(x\mid\mathcal M_i)
\sum_{j\neq i}
\exp\left[
-\frac12\left(u_i(x)-u_j(x)\right)
\right]dx.
\]

Define

\[
p_i(x)\equiv p(x\mid\mathcal M_i).
\]

Adding the diagonal terms \(i=j\), which contribute only

\[
\sum_i p_i(x)
\]

and are independent of \(\mathbf u\), we may equivalently minimize

\[
\widetilde{\mathcal R}[\mathbf u]
=
\int
\left(
\sum_i p_i(x)e^{-u_i(x)/2}
\right)
\left(
\sum_j e^{u_j(x)/2}
\right)dx.
\]

For compactness, define

\[
A(x)
=
\sum_i p_i(x)e^{-u_i(x)/2},
\qquad
B(x)
=
\sum_i e^{u_i(x)/2}.
\]

Then

\[
\widetilde{\mathcal R}[\mathbf u]
=
\int A(x)B(x)\,dx.
\]

Again we follow variation of the functional:

\[
\delta\widetilde{\mathcal R}[\mathbf u;\boldsymbol\eta]
=
\left.
\frac{d}{d\varepsilon}
\widetilde{\mathcal R}
[\mathbf u+\varepsilon\boldsymbol\eta]
\right|_{\varepsilon=0}.
\]

Using the product rule, we first apply this to each term separately:

\[
A_\varepsilon(x)
=
\sum_i
p_i(x)e^{-[u_i(x)+\varepsilon\eta_i(x)]/2},
\]

and therefore

\[
\left.
\frac{dA_\varepsilon}{d\varepsilon}
\right|_{\varepsilon=0}
=
-\frac12
\sum_i
p_i(x)e^{-u_i(x)/2}\eta_i(x).
\]

Similarly for the second term,

\[
B_\varepsilon(x)
=
\sum_i
e^{[u_i(x)+\varepsilon\eta_i(x)]/2},
\]

so that

\[
\left.
\frac{dB_\varepsilon}{d\varepsilon}
\right|_{\varepsilon=0}
=
\frac12
\sum_i
e^{u_i(x)/2}\eta_i(x).
\]

Now according to the product rule we can write

\[
\begin{aligned}
\delta\widetilde{\mathcal R}
&=
\int
\left[
\delta A(x)\,B(x)
+
A(x)\,\delta B(x)
\right]dx \\
&=
\frac12
\int
\sum_{k=1}^N
\left[
A(x)e^{u_k(x)/2}
-
B(x)p_k(x)e^{-u_k(x)/2}
\right]
\eta_k(x)\,dx.
\end{aligned}
\]

To find the extremum we set the variation to 0:
\[
\delta\widetilde{\mathcal R}
[\mathbf u^*;\boldsymbol\eta]
=
0.
\]

By the fundamental lemma of the calculus of variations, this implies, for
each \(k\),

\[
A(x)e^{u_k^*(x)/2}
-
B(x)p_k(x)e^{-u_k^*(x)/2}
=
0.
\]

Thus

\[
A(x)e^{u_k^*(x)/2}
=
B(x)p_k(x)e^{-u_k^*(x)/2}.
\]

Multiplying by \(e^{u_k^*(x)/2}\) gives

\[
A(x)e^{u_k^*(x)}
=
B(x)p_k(x),
\]

and therefore

\[
e^{u_k^*(x)}
=
p_k(x)\frac{B(x)}{A(x)}.
\]

Taking logarithms,

\[
u_k^*(x)
=
\log p_k(x)
+
\log\frac{B(x)}{A(x)}.
\]

The second term is independent of the model index \(k\). Hence the stationary
solutions have the form

\[
\boxed{
u_k^*(x)
=
\log p(x\mid\mathcal M_k)+c(x)
},
\]

where \(c(x)\) is a common, non-identifiable offset.

Taking the difference between any two outputs eliminates this common offset:

\[
u_i^*(x)-u_j^*(x)
=
\log p(x\mid\mathcal M_i)
-
\log p(x\mid\mathcal M_j).
\]

Therefore,

\[
\boxed{
u_i^*(x)-u_j^*(x)
=
\log
\frac{p(x\mid\mathcal M_i)}
     {p(x\mid\mathcal M_j)}
=
\log K_{ij}(x)
},
\]

and removing the shorthand notation

\[
\boxed{
\mathcal J_\alpha(f_i^*(x))
-
\mathcal J_\alpha(f_j^*(x))
=
\log K_{ij}(x)
}.
\]

To verify that this stationary family gives the global minimum, note that at
each \(x\),

\[
A(x)B(x)
=
\left(
\sum_i p_i(x)e^{-u_i(x)/2}
\right)
\left(
\sum_i e^{u_i(x)/2}
\right).
\]

By the Cauchy--Schwarz inequality,

\[
A(x)B(x)
\geq
\left(
\sum_i\sqrt{p_i(x)}
\right)^2.
\]

Equality holds if and only if

\[
\sqrt{p_i(x)}\,e^{-u_i(x)/4}
=
\lambda(x)e^{u_i(x)/4}
\]

for every \(i\), which is equivalent to

\[
u_i(x)=\log p_i(x)-2\log\lambda(x)\equiv \log p_i(x)+c(x).
\]

Thus the variational stationary family obtained above attains the pointwise
lower bound of the population-risk integrand and is therefore a global
minimizer. The common function \(c(x)\) remains undetermined because the loss
depends only on pairwise differences \(u_i-u_j\).

\end{document}